\documentclass{article}
\usepackage{ijcai15}
\usepackage{times}
\usepackage{graphicx}
\usepackage{subfigure}
\usepackage{algorithm}
\usepackage{algorithmic}
\usepackage{amsmath}
\usepackage{amsfonts}
\usepackage{bm}

\title{Thompson Sampling for Budgeted Multi-armed Bandits}
\author{Yingce Xia$^{1,}$\thanks{This work was done when the first two authors were interns at Microsoft Research.}, Haifang Li$^{2,*}$, Tao Qin$^3$, Nenghai Yu$^1$ \and Tie-Yan Liu$^3$\\
	$^1$University of Science and Technology of China, Hefei, China\\
	$^2$ University of Chinese Academy of Sciences, Beijing, China\quad$^3$Microsoft Research, Beijing, China\\
	yingce.xia@gmail.com, lihaifang@amss.ac.cn, \{taoqin, tyliu\}@microsoft.com, ynh@ustc.edu.cn}

\begin{document}

\maketitle
\newtheorem{theorem}{Theorem}
\newtheorem{lemma}[theorem]{Lemma}
\newtheorem{fact}{Fact}
\newtheorem{definition}[theorem]{Definition}
\newtheorem{remark}[theorem]{Remark}

\def\QEDclosed{\mbox{\rule[0pt]{1.3ex}{1.3ex}}}
\def\QEDopen{{\setlength{\fboxsep}{0pt}\setlength{\fboxrule}{0.2pt}\fbox{\rule[0pt]{0pt}{1.3ex}\rule[0pt]{1.3ex}{0pt}}}}
\newenvironment{proof}[1][Proof.]{\begin{trivlist}
\item[\hskip \labelsep {\emph{#1}}]}{\end{trivlist}}
\newenvironment{proofsketch}[1][Proof sketch.]{\begin{trivlist}
\item[\hskip \labelsep {\bfseries #1}]}{\QEDclosed\end{trivlist}}

\newcommand{\myeqref}[1]{\eqref{#1}}
\newcommand{\suboptpull}[1]{n_{#1 , t(B)}}
\def\myexp{\mathbb{E}}
\def\myprob{\mathbb{P}}

\begin{abstract}
	Thompson sampling is one of the earliest randomized algorithms for multi-armed bandits (MAB). In this paper, we extend the Thompson sampling to Budgeted MAB, where there is random cost for pulling an arm and the total cost is constrained by a budget. We start with the case of Bernoulli bandits, in which the random rewards (costs) of an arm are independently sampled from a Bernoulli distribution. To implement the Thompson sampling algorithm in this case, at each round, we sample two numbers from the posterior distributions of the reward and cost for each arm, obtain their ratio, select the arm with the maximum ratio, and then update the posterior distributions. We prove that the distribution-dependent regret bound of this algorithm is $O(\ln B)$, where $B$ denotes the budget. By introducing a Bernoulli trial, we further extend this algorithm to the setting that the rewards (costs) are drawn from general distributions, and prove that its regret bound remains almost the same. Our simulation results demonstrate the effectiveness of the proposed algorithm.
\end{abstract}

\section{Introduction}\label{sec:intro}
The multi-armed bandit (MAB) problem, a classical sequential decision problem in an uncertain environment, has been widely studied in the literature \cite{lai1985asymptotically,UCB1}.  Many real world applications can be modeled as MAB problems, such as news recommendation \cite{li2010contextual} and channel allocation \cite{gai2010learning}. Previous studies on MAB can be classified into two categories: one focuses on designing  algorithms to find a policy that can maximize the cumulative expected reward, such as UCB1 \cite{UCB1}, UCB-V \cite{ucbv}, MOSS \cite{moss}, KL-UCB \cite{garivier2011kl} and Bayes-UCB \cite{kaufmann2012bayesian}; the other aims at studying the sample complexity to reach a specific accuracy, such as \cite{bubeck2009pure,yu2013sample}.

Recently, a new setting of MAB, called \emph{budgeted MAB}, was proposed to model some new Internet applications, including online bidding optimization in sponsored search \cite{175457xx,GBLS_long} and on-spot instance bidding in cloud computing \cite{agmon2013deconstructing,ardagna2011game}. In budgeted MAB, pulling an arm receives both a random reward and a random cost, drawn from some unknown distributions. The player can keep pulling the arms until he/she runs out of budget $B$. A few algorithms have been proposed to solve the Budgeted MAB problem.  For example, in \cite{epsilon_first},  an $\epsilon$-first algorithm was proposed which first spends $\epsilon B$ budget on pure explorations, and then keeps pulling the arm with the maximum empirical reward-to-cost ratio. It was proven that the $\epsilon$-first  algorithm has a regret bound of $O(B^{\frac{2}{3}})$. KUBE \cite{KUBE} is another algorithm for budgeted MAB, which solves an integer linear program at each round, and then converts the solution to the probability of each arm to be pulled at the next round. A limitation of the $\epsilon$-first and KUBE algorithms lies in that they assume the cost of each arm to be deterministic and fixed, which narrows their application scopes. In \cite{ding2013multi}, the setting was considered that the cost of each arm is drawn from an unknown discrete distribution and two algorithms UCB-BV$1$/BV$2$ were designed. A limitation of these algorithms is that they require additional information about the minimum expected cost of all the arms, which is not available in some applications.

Thompson sampling \cite{thompson1933likelihood} is one of the earliest randomized algorithms for MAB, whose main idea is to choose an arm according to its posterior probability to be the best arm. In recent years, quite a lot of studies have been conducted on Thompson sampling, and good performances have been achieved in practical applications \cite{chapelle2011empirical}. It is proved in \cite{kaufmann2012thompson} that Thompson sampling can reach the lower bound of regret given in \cite{lai1985asymptotically} for Bernoulli bandits. Furthermore, problem-independent regret bounds were derived in \cite{export:191857} for Thompson sampling with Beta and Gaussian priors.

Inspired by the success of Thompson sampling in classical MAB, two natural questions arise regarding its extension to budgeted MAB problems: (i) How can we adjust Thompson sampling so as to handle budgeted MAB problems? (ii) What is the performance of Thompson sampling in theory and in practice? In this paper, we try to provide answers to these two questions.

\emph{Algorithm}: We propose a refined Thompson sampling algorithm that can be used to solve the budgeted MAB problems. While the optimal policy for budgeted MAB could be very complex (budgeted MAB can be viewed as a stochastic version of the knapsack problem in which the value and weight of the items are both stochastic), we prove that, when the reward and cost per pulling are supported in $[0,1]$ and the budget is large, we can achieve the almost optimal reward by always pulling the optimal arm (associated with the maximum expected-reward-to-expected-cost ratio). With this guarantee, our proposed algorithm targets at pulling the optimal arm as frequently as possible. We start with Bernoulli bandits, in which the random rewards (costs) of an arm are independently sampled from a Bernoulli distribution. We design an algorithm which (1) uses beta distribution to model the priors of the expected reward and cost of each arm, and (2) at each round, samples two numbers from the posterior distributions of the reward and cost for each the arm, obtains their ratio, selects the arm with the maximum ratio, and then updates the posterior distributions. We further extend this algorithm to the setting that the rewards (costs) are drawn from general distributions by introducing Bernoulli trials. 

\emph{Theoretical analysis}: We prove that our proposed algorithm can achieve a distribution-dependent regret bound of $O(\ln B)$, with a tighter constant before $\ln B$ than existing algorithms (e.g., the two algorithms in \cite{ding2013multi}). To obtain this regret bound, we first show that it suffices to bound the expected pulling times of all the suboptimal arms (whose expected-reward-to-expected-cost ratios are not maximum). To this end, for each suboptimal arm, we define two gaps, the $\delta$-ratio gap and the $\epsilon$-ratio gap, which compare its expected-reward-to-expected-cost ratio to that of the optimal arm. Then by introducing some intermediate events, we can decompose the expected pulling time of a suboptimal arm $i$ into several terms, each of which depends on only the reward or only the cost. After that, we can bound each term by the concentration inequalities and two gaps with careful derivations.

To our knowledge, it is the first time that Thompson sampling is applied to the budgeted MAB problem. We conduct a set of numerical simulations with different rewards/costs distributions and different number of arms. The simulation results demonstrate that our proposed algorithm is much better than several baseline algorithms.

\section{Problem Formulation}
In this section, we give a formal definition to the budgeted MAB problem.

In budgeted MAB, we consider a slot machine with $K$ arms ($K\ge2$). \footnote{Denote the set $\{1,2,\cdots,K\}$ as $[K]$.} At round $t$, a player pulls an arm $i\in[K]$, receives a random reward $r_i(t)$, and pays a random cost $c_i(t)$ until he runs out of his budget $B$, which is a positive integer.  Both the reward $r_i(t)$ and the cost $c_i(t)$ are supported on $[0,1]$. For simplicity and following the practice in previous works, we make a few assumptions on the rewards and costs: (i) the rewards of an arm are independent of its costs; (ii) the rewards and costs of an arm are independent of other arms; (iii) the rewards and costs of the same arm at different rounds are independent and identically distributed.

We denote the expected reward and cost of arm $i$ as $\mu_i^r$ and $\mu_i^c$ respectively. W.l.o.g., we assume $\forall i \in[K]$, $\mu^r_i>0$, $\mu^c_i>0$, and $\arg\max_{i\in[K]}\frac{\mu_i^r}{\mu_i^c}=1$. We name arm $1$ as the optimal arm and the other arms as suboptimal arms.

Our goal is to design algorithms/policies for budgeted MAB with small pseudo-regret, which is defined as follows:
\begin{small}
	\begin{equation}
		\textrm{Regret} = R^* - \myexp\sum_{t=1}^{T_B}r_t,
		\label{eq:regret_time}
	\end{equation}
\end{small}
where $R^*$ is the expected reward of the optimal policy (the policy that can obtain the maximum expected reward given the reward and cost distributions of each arm), $r_t$ is the reward received by an algorithm at round $t$, $T_B$ is the stopping time of the algorithm, and the expectation is taken w.r.t. the randomness of the algorithm, the rewards (costs), and the stopping time.

Please note that it could be very complex to obtain the optimal policy for the budgeted MAB problem (under the condition that the reward and cost distributions of each arm are known). Even for its degenerated case, where the reward and cost of each arm are deterministic, the problem is known to be NP-hard (actually in this case the problem becomes an unbounded knapsack problem \cite{martello1990knapsack}).  Therefore, generally speaking, it is hard to calculate $R^*$ in an exact manner.

However, we find that it is much easier to approximate the optimal policy and to upper bound $R^*$. Specifically, when the reward and cost per pulling are supported in $[0,1]$ and $B$ is large, always pulling the optimal arm could be very close to the optimal policy. For Bernoulli bandits, since there is no time restrictions on pulling arms, one should try to always pull arm $1$ so as to fully utilize the budget\protect\footnote{This is inspired by the greedy heuristic for the knapsack problem \cite{fisher1980worst}, i.e., at each round, one selects the item with the maximum value-to-weight ratio. Although there are many approximation algorithms for the knapsack problem like the total-value greedy heuristic \cite{kohli1992total} and the FPTAS \cite{vazirani2001approximation}, under our budgeted MAB setting, we find that they will not bring much benefit on tightening the bound of $R^*$.}. For the general bandits, the situation is a little more complicated and pulling arm $1$ will result in a suboptimiality of at most $2\frac{\mu_1^r}{\mu_1^c}$. These results are summarized in Lemma \ref{lemma:optimalPolicy2}, together with upper bounds on $R^*$. The proof of Lemma \ref{lemma:optimalPolicy2} can be found at Appendix \ref{app:proof_lemma_opt}.

\begin{lemma}
	When the reward and cost per pulling are supported in $[0,1]$, for Bernoulli bandits, we have $R^*=\frac{\mu_1^r}{\mu_1^c}B$ and the optimal policy is exactly always pulling arm $1$; for general bandits, we have $R^*\le\frac{\mu_1^r}{\mu_1^c}(B + 1)$, and the suboptimality of always pulling arm $1$ (as compared to the optimal policy) is at most $2\frac{\mu_1^r}{\mu_1^c}$.
	\label{lemma:optimalPolicy2}
\end{lemma}
For any $i\ge2$, define $T_i$ as the pulling time of arm $i$ when running out of budget. Denote the difference of the expected-reward-to-expected-cost ratio between the optimal arm $1$ and a suboptimal arm $i(\ge2)$ as $\Delta_i$:
\begin{small}
	\begin{equation}
		\Delta_i = \frac{\mu_1^r}{\mu_1^c} - \frac{\mu_i^r}{\mu_i^c},\quad\forall i \ge 2.
		\label{def:Delta}
	\end{equation}
\end{small}
Lemma \ref{lemma:RegretBound_GeneralDistribution} relates the regret to $T_i$ and $\Delta_i$ ($i\ge2$). It is useful when we analyze the regret of a pulling algorithm.
\begin{lemma}
	For Bernoulli bandits, we have
	\begin{small}
		\begin{equation}
			\textrm{Regret} = \sum_{i=2}^{K}\mu_i^c\Delta_i\myexp\{T_{i}\}.
			\label{eq:RegretBound_BernoulliDistibutrion}
		\end{equation}
	\end{small}	
	For general bandits, we have
	\begin{small}
		\begin{equation}
			\textrm{Regret} \leq  2\frac{\mu_1^r}{\mu_1^c}+\sum_{i=2}^{K}\mu_i^c\Delta_i\myexp\{T_i\}.
			\label{eq:AGenaralOfTheBerBandit}
		\end{equation}
	\end{small}	
	\label{lemma:RegretBound_GeneralDistribution}	
\end{lemma}
The intuition behind Lemma \ref{lemma:RegretBound_GeneralDistribution} is as follows.  As aforementioned, for Bernoulli bandits, the optimal policy is to always pull arm 1. If one pulls a suboptimal arm $i$ ($>1$) for $T_i$ times, then he/she will lose some rewards. Specifically, the expected budget spent on arm $i$ is $\mu_i^cT_i$, and if he/she spent such budget on the optimal arm $1$, he/she can get $\mu^c_i\Delta_iT_i$ extra reward. For general bandits, always pulling arm $1$ might not be optimal (see Lemma \ref{lemma:optimalPolicy2}) -- actually it leads to a regret at most $\frac{2\mu_1^r}{\mu_1^c}$. Therefore, we need to add an extra term $\frac{2\mu_1^r}{\mu_1^c}$ to the result for Bernoulli bandits. The proof of Lemma \ref{lemma:RegretBound_GeneralDistribution} can be found at Appendix \ref{subsec:proof_pulltoregret} and \ref{app:proof_pull_regret}.

\section{Budgeted Thompson Sampling} \label{sec:AllAlgs}
In this section, we first show how Thompson sampling can be extended to handle budgeted MAB with Bernoulli distributions, and then generalize the setting to general distributions. For ease of reference, we call the corresponding algorithm \emph{Budgeted Thompson Sampling (BTS)}.

First, the BTS algorithm for the budgeted Bernoulli bandits is shown in Algorithm \ref{alg:TS}. In the algorithm, $S_i^r(t)$ denotes the times that the player receives reward $1$ from arm $i$ before (excluding) round $t$, $S_i^c(t)$ denotes the times that the player pays cost $1$ for pulling arm $i$ before (excluding) round $t$, and $Beta(\cdot,\cdot)$ denotes the beta distribution. Please note that we use beta distribution as a prior in Algorithm \ref{alg:TS} because it is the conjugate distribution of the binomial distribution: If the prior is a $Beta(\alpha,\beta)$, after a Bernoulli experiment, the posterior distribution is either $Beta(\alpha+1,\beta)$ (if the trial is a success) or $Beta(\alpha,\beta+1)$ (if the trial is a failure).

In the original Thompson sampling algorithm, one draws a sample from the posterior Beta distribution for the reward of each arm, pulls the arm with the maximum sampled reward, receives a reward, and then updates the reward distribution based on the received reward. In Algorithm \ref{alg:TS}, in addition to sampling rewards, we also sample costs for the arms at the same time, pull the arm with the maximum sampled reward-to-cost ratio, receive both the reward and cost, and then update the reward distribution and cost distribution.

As compared to KUBE \cite{KUBE}, Algorithm \ref{alg:TS} does not need to solve a complex integer linear program. As compared to the UCB-style algorithms like fractional KUBE \cite{KUBE} and UCB-BV1 \cite{ding2013multi}, Algorithm \ref{alg:TS} does not need carefully designed confidence bounds. As can be seen, BTS only simply chooses one out of the $K$ arms according to their posterior probabilities to be the best arm, which is an intuitive, easy-to-implement, and efficient approach.
\renewcommand{\algorithmicrequire}{\textbf{Input:}}
\renewcommand{\algorithmicensure} {\textbf{Output:} }
\begin{algorithm}[!htb]
	\caption{Budgeted Thompson Sampling (BTS)}
	\begin{algorithmic}[1]
		{
			\STATE For each arm $i \in [K]$, set $S_i^r(1) \leftarrow 0$, $F_i^r(1) \leftarrow 0$, $S_i^c(1) \leftarrow 0$, and $F_i^c(1) \leftarrow 0$ ;\
			\STATE Set $B_1\leftarrow B$; $t\leftarrow1$;\
			\WHILE{$B_t > 0$}
			\STATE For each arm $i \in [K]$, sample $\theta_i^r(t)$ from $Beta(S_i^r(t) + 1 , F_i^r(t) + 1)$ and sample $\theta_i^c(t)$ from $Beta(S_i^c(t) + 1 , F_i^c(t) + 1)$;\
			\STATE Pull arm $I_t = \arg \max_{i \in [K]}\frac{\theta_i^r(t)}{\theta_i^c(t)}$; receive reward $r_t$; pay cost $c_t$; update $B_{t+1} \leftarrow B_t - c_t$;\
			\STATE \emph{For Bernoulli bandits}, $\tilde{r}\leftarrow r_t,\tilde{c}\leftarrow c_t$; \emph{for general bandits}, sample $\tilde{r}$ from $\mathcal{B}(r_t)$ and sample $\tilde{c}$ from $\mathcal{B}(c_t)$;\
			\STATE $S_{I_t}^r(t+1)\leftarrow S_{I_t}^r(t) + \tilde{r}$; $F_{I_t}^r(t+1) \leftarrow F_{I_t}^r(t) + 1-\tilde{r}$;\
			\STATE $S_{I_t}^c(t+1)\leftarrow S_{I_t}^c(t)+ \tilde{c}$; $F_{I_t}^c(t+1) \leftarrow F_{I_t}^c(t) + 1-\tilde{c}$; \
			\STATE $\forall j \ne I_t$, $S_j^r(t+1) \leftarrow S_j^r(t)$, $F_j^r(t+1) \leftarrow F_j^r(t)$, $S_j^c(t+1) \leftarrow S_j^c(t)$, $F_j^c(t+1) \leftarrow F_j^c(t)$;\
			\STATE Set $t \leftarrow t + 1$.\
			\ENDWHILE
		}
	\end{algorithmic}
	\label{alg:TS}
\end{algorithm}

By leveraging the idea proposed in \cite{export:166345}, we can modify the BTS algorithm for Bernoulli bandits and make it work for bandits with general reward/cost distributions. In particular, with general distributions, the reward $r_t$ and cost $c_t$ (in Step 5) at round $t$ become real numbers in $[0,1]$. We introduce a Bernoulli trial in Step 6:  Set $\tilde{r} \leftarrow \mathcal{B}(r_t)$ and $\tilde{c} \leftarrow \mathcal{B}(c_t)$, in which $\mathcal{B}(r_t)$ is a Bernoulli test with success probability $r_t$ and so is $\mathcal{B}(c_t)$. Now $S_i^r(t)$ and $S_i^c(t)$ represent the number of success Bernoulli trials for the reward and cost respectively. Then we can use $\tilde{r}$  and $\tilde{c}$ to update $S_i^r(t)$ and $S_i^c(t)$ accordingly.

\section{Regret Analysis}
In this section, we analyze the regret of our proposed BTS algorithm. We start with Bernoulli bandits and then generalize the results to general bandits. We give a proof sketch in the main text and details can be found in the appendix.

In a classical MAB, the player only needs to explore the expected reward of each arm, however, in a budgeted MAB the player also needs to explore the expected cost simultaneously. Therefore, as compared with \cite{export:166345}, our regret analysis will heavily depends on some quantities related to the reward-to-cost ratio (such as the two gaps defined below). 

For an arm $i(\ge2)$ and a given $\gamma \in (0,1)$,we define 
\begin{small}
	$$\delta_i(\gamma) = \frac{\gamma\mu_i^c\Delta_i}{\frac{\mu_1^r}{\mu_1^c} + 1}, \quad \epsilon_i(\gamma)= \frac{(1-\gamma)\mu_1^c\Delta_i}{\frac{\mu_i^r}{\mu_i^c} + 1}.
	$$
\end{small}
It is easy to verify the following equation for any $i\ge2$.
\begin{small}
	$$ \frac{\mu_i^r + \delta_i(\gamma)}{\mu_i^c - \delta_i(\gamma)} = \frac{\mu_1^r - \epsilon_i(\gamma)}{\mu_1^c + \epsilon_i(\gamma)}$$
\end{small}
For ease of reference, $\forall i \ge 2$, we call $\delta_i(\gamma)$ the $\delta$-\emph{ratio gap} between the optimal arm and a suboptimal arm $i$, and $\epsilon_i(\gamma)$ the $\epsilon$-\emph{ratio gap}. In the remaining part of this section, we simply write $\epsilon_i(\gamma)$ as $\epsilon_i$ when the context is clear and there is no confusion.

The following theorem says that BTS achieves a regret bound of $O(\ln(B))$ for both Bernoulli and general bandits:
\begin{theorem}
	$\forall \gamma \in (0,1)$, for both Bernoulli bandits and general bandits, the regret of the BTS algorithm can be upper bounded as below.
	\begin{small}
		\begin{equation*}
			\textrm{Regret}  \leq
			\sum_{i=2}^{K}\Big\{\frac{2\ln B}{\gamma^2\mu_i^c\Delta_i}\bigg(\frac{\mu_1^r}{\mu_1^c}+1\bigg)^2 + \Phi_i(\gamma)\Big\}+O\bigg(\frac{K}{\gamma^2}\bigg),
		\end{equation*}
	\end{small}
	in which $\Delta_i$ is defined in Eqn. \myeqref{def:Delta} and {\small{$\Phi_i(\gamma)$}} is defined as
	\begin{small}
		\begin{equation}
			\left\{
			\begin{aligned}
				& O\Big(\frac{1}{\epsilon_i^4(\gamma)}\Big), &&\textrm{if }\mu_1^c+\epsilon_i(\gamma)\ge1;\\
				& O\Big(\frac{1}{\epsilon_i^6(\gamma)(1 - \mu^c_1-\epsilon_i(\gamma))}\Big),&&\textrm{if }\mu_1^c+\epsilon_i(\gamma)<1.
			\end{aligned}
			\right.
			\label{eq:BigPhi}
		\end{equation}
	\end{small}	
	\label{thm:bound_TS_asymp}
\end{theorem}

We first prove Theorem \ref{thm:bound_TS_asymp} holds for Bernoulli bandits in Section \ref{subsec:proof_bernoulli} and then extend the result for general bandits in Section \ref{subsec:generalExternsion}	 .

\subsection{Analysis for Bernoulli Bandits}\label{subsec:proof_bernoulli}
First, we describe the high-level idea of how to prove the theorem. According to Lemma \ref{lemma:RegretBound_GeneralDistribution}, to upper bound the regret of BTS, it suffices to bound $\myexp\{T_i\}$ $\forall i\ge2$. For a suboptimal arm $i$, $\myexp\{T_i\}$ can be decomposed into the sum of a constant and the probabilities of two kinds of events (see \myeqref{eq:basicDecomposition}). The first kind of event is related to the $\delta$-ratio gap $\delta_i(\gamma)$, and its probability can be bounded by leveraging concentrating inequalities and the relationship between the binomial distribution and the beta distribution. The second one is related to the $\epsilon$-ratio gap $\epsilon_i(\gamma)$,  according to which the probability of the event related to arm $i$ can be converted to that related to the optimal arm $1$. To bound the probability of the second kind of event, we need some complicated derivations, as shown in the later part of this subsection. 

Then, we define some notations and intermediate variables, which will be used in the proof sketch.

$n_{i,t}$ denotes the pulling time of arm $i$ before (excluding) round $t$; $I_t$ denotes the arm pulled at round $t$; $\bm{1}\{\cdot\}$ is the indicator function; $\mu_{\textrm{min}}^c = \min_{i\in[K]}\{\mu_i^c\}$; $H_{t-1}$ denotes the history until round $t-1$, including the arm pulled from round 1 to $t-1$, and the rewards/costs received at each round; $\theta_i(t)$ denotes the ratio $\frac{\theta_i^r(t)}{\theta_i^c(t)}$ $\forall i\in[K]$  where $\theta_i^r(t)$ and $\theta^c_i(t)$ are defined in Step 4 of Algorithm \ref{alg:TS}; $B_t$ denotes the budget left at the beginning of round $t$; $E^\theta_i(t)$ denotes the event that given $\gamma\in(0,1)$, $\theta_i(t)\leq \frac{\mu_i^r + \delta_i(\gamma)}{\mu_i^c - \delta_i(\gamma)}$ $\forall i > 1$;  the probability $p_{i,t}$ denotes $\myprob\{\theta_1(t)>\frac{\mu_1^r -\epsilon_i(\gamma)}{\mu_1^c + \epsilon_i(\gamma)}|H_{t-1}, B_t > 0\}$ given $\gamma\in(0,1)$; $\overline{event}$ denotes the ``$event$'' does not hold.

After that, we give the proof sketch as follows, which can be partitioned into four steps.\\
\emph{\textbf{Step 1}: Decompose $\myexp\{T_i\}$ ($i>1$)}.

It can be shown that the pulling time of a suboptimal arm $i$ can be decomposed into three parts: a constant invariant to $t$ and the probabilities of two kinds of events:
\begin{small}
	\begin{align}
		\myexp\{T_i\} & \leq  \lceil L_i \rceil + \sum_{t=1}^{\infty}\myprob\{\overline{E_i^\theta(t)},n_{i,t}\ge \lceil L_i \rceil, B_t > 0 \}\nonumber\\
		&+ \sum_{t=1}^{\infty}\myprob\{I_t = i,E_i^\theta(t),B_t>0\},	\label{eq:basicDecomposition}
	\end{align}
	where $L_i = \frac{2\ln B}{\delta_i^2(\gamma)}$.
\end{small}
The derivations of \myeqref{eq:basicDecomposition} is left in Appendix \ref{app:coredecomposition}. Note that $L_i$ depends on $\gamma$. We omit the $\gamma$ when there is no confusion throughout the context. We then bound the probabilities of  the two kinds of events  in the next two steps.\\
\emph{\textbf{Step 2:} Bound $\sum_{t=1}^{\infty}\myprob\{\overline{E_i^\theta(t)},n_{i,t}\ge \lceil  L_i \rceil, B_t > 0 \}$.} \\
Define two new events: $\forall i \ge 2$ and $t\ge1$, 
\begin{equation*}
	(\textrm{I})\,E_i^r(t):\,\theta_i^r(t) \leq \mu_i^r + \delta_i(\gamma) \textrm{; }
	(\textrm{II})\,E_i^c(t):\,\theta_i^c(t) \geq \mu_i^c -\delta_i(\gamma).
\end{equation*}
If $\overline{E_i^\theta(t)}$ holds, at least one event of $\overline{E^r_i(t)}$ and $\overline{E^c_i(t)}$ holds. Therefore, we have
\begin{small}
	\begin{align}
		&\myprob\{\overline{E_i^\theta(t)},n_{i,t}\geq \lceil L_i\rceil|B_t>0\} 
		\leq \myprob\{\overline{E_i^r(t)},n_{i,t}\geq \lceil L_i\rceil|B_t>0\} \nonumber\\
		&+ \myprob\{\overline{E_i^c(t)},n_{i,t}\geq \lceil L_i\rceil|B_t>0\}.\label{eq:pulledNumOfSubOptimalArm_i}
	\end{align}
\end{small}
Intuitively, when $n_{i,t}$ is large enough, $\theta_i^r(t)$ and $\theta_i^c(t)$ should be very close to $\mu_i^r$ and $\mu_i^c$ respectively. Then, both $\overline{E_i^r(t)}$ and $\overline{E_i^c(t)}$ will be low-probability events. Mathematically, $\forall \gamma \in(0,1)$, the two terms in the right-hand side of \myeqref{eq:pulledNumOfSubOptimalArm_i} could be bounded as follows, by considering the relationship between the binomial distribution and the beta distribution.
\begin{small}
	\begin{align}
		&\myprob\{\overline{E_i^r(t)} , n_{i,t} \geq \lceil L_i \rceil | B_t > 0\} \leq \frac{7}{B\delta_i^2(\gamma)}.   \label{eq:bound_firstEqn}\\
		&\myprob\{\overline{E_i^c(t)} , n_{i,t} \geq \lceil L_i\rceil  | B_t > 0\} \leq \frac{28}{B\delta_i^2(\gamma)}. \label{eq:bound_2ndEqn}		
	\end{align}
\end{small}
The proof of \myeqref{eq:bound_firstEqn} and \myeqref{eq:bound_2ndEqn} can be found at Appendix \ref{app:proof_low:AA} and  \ref{subapp:proofoflowprobevent12}.
As a result, we have
{\small{
		$$\myprob\{\overline{E_i^\theta(t)},n_{i,t}\geq \lceil L_i \rceil|B_t>0\}\leq \frac{35}{B\delta_i^2(\gamma)}.$$
	}}
	One can also verify that {$\sum_{t=1}^{\infty}\myprob\{B_t > 0\}$} is bounded by
	\begin{small}
		\begin{equation}
			\begin{aligned}
				& \frac{1}{\mu^c_{\min}}\sum_{t=1}^{\infty}\sum_{i=1}^{K}\myexp\{c_i(t)\bm{1}\{I_t=i\}|B_t>0\}\myprob\{B_t>0\}\leq \frac{B}{\mu^c_{\min}},
				\label{eq:internal_bound_t(B)}
			\end{aligned}
		\end{equation}
	\end{small}
	where $c_i(t)$ is the cost of arm $i$ at round $t$.
	
	Therefore, we  obtain that
	\begin{small}
		\begin{equation}
			\sum_{t=1}^{\infty}\myprob\{\overline{E_i^\theta(t)},n_{i,t}\ge \lceil L_i \rceil,B_t > 0 \}\leq\frac{35}{\delta_i^2(\gamma)\mu^c_{\min}}.
			\label{eq:boundthefirsttermAB}
		\end{equation}
	\end{small}
	\emph{\textbf{Step 3}: Bound $\sum_{t=1}^{\infty}\myprob\{I_t=i,E_i^\theta(t),B_t > 0\}$.}\\
	Let $\tau_k$ ($k\ge0$) denote the round that arm $1$ has been pulled for the $k$-th time and define $\tau_0 = 0$. $\forall i \ge 2$ and $\forall t \ge 1$, $p_{i,t}$ is only related to the pulling history of arm $1$, thus $p_{i,t}$ will not change between $\tau_k+1$ and $\tau_{k+1}$, $\forall k \ge 0$. With some derivations, we can get that
	\begin{small}
		\begin{align}
			&\sum_{t=1}^{\infty}\myprob\{I_t=i,E_i^\theta(t),B_t>0\} \leq\sum_{k=0}^{\infty}\Big(\myexp\Big\{\frac{1}{p_{i,\tau_k+1}}\Big\}-1\Big).\label{eq:bound_arm1_hard}
		\end{align}
	\end{small}
\myeqref{eq:bound_arm1_hard} bridges the probability of an event related to arm $1$ and that related to arm $i$ $(i\ge2)$. Derivations of  \myeqref{eq:bound_arm1_hard} can be found at Appendix \ref{app:veryhardderivation}.
To further decompose the r.h.s. of \myeqref{eq:bound_arm1_hard}, define the following two probabilities which are related to the $\epsilon$-ratio gap between arm $1$ and arm $i$:
	\begin{small}
		\begin{align}
			&p_{i,t}^r=\myprob\{\theta_1^r(t) \geq \mu_1^r - \epsilon_i(\gamma)|H_{t-1}\},\nonumber\\
			&p_{i,t}^c=\myprob\{\theta_1^c(t) \leq \mu_1^c+ \epsilon_i(\gamma)|H_{t-1}\}.\nonumber
		\end{align}
	\end{small}
	Since the reward of an arm is independent of its cost, we can verify $p_{i,t} \geq p_{i,t}^rp_{i,t}^c$ and then get
	\begin{small}
		\begin{equation}
			\myexp\Big\{\frac{1}{p_{i,\tau_k+1}}\Big\}\leq\myexp\Big\{\frac{1}{p^r_{i,\tau_k+1}}\Big\}\myexp\Big\{\frac{1}{p_{i,\tau_k+1}^c}\Big\}.
			\label{eq:core_bound_ts}
		\end{equation}
	\end{small}
	According to \myeqref{eq:bound_arm1_hard} and \myeqref{eq:core_bound_ts}, $\sum_{t=1}^{\infty}\myprob\{I_t=i,E_i^\theta(t),B_t > 0\}$ can be bounded by the sum of the right-hand side of \myeqref{eq:core_bound_ts} over index $k$ from $0$ to infinity, which is related to the pulling time of arm $1$ and its $\epsilon$-ratio gaps. 
	
	It is quite intuitive that when arm $1$ is played for enough times, $\theta_1^r(t)$ and $\theta_1^c(t)$ will be very close to $\mu_1^r$ and $\mu_1^c$ respectively. That is, probabilities $p^r_{i,\tau_k+1}$ and $p^c_{i,\tau_k+1}$ will be close to $1$, and so will their reciprocals.
	To mathematically characterize $p^r_{i,\tau_k+1}$ and $p^c_{i,\tau_k+1}$, we define some notations as follows, which are directly or indirectly related to the $\epsilon$-ratio gap:  $y_i = \mu_1^r - \epsilon_i$, $z_i=\mu_1^c + \epsilon_i$, $R_{1,i}=\frac{\mu_1^r(1-y_i)}{y_i(1-\mu_1^r)}$, $R_{2,i}=\frac{\mu_1^c(1-z_i)}{z_i(1-\mu_1^c)}$, $D_{1,i} = y_i\ln(\frac{y_i}{\mu_1^r}) + (1 - y_i)\ln(\frac{1-y_i}{1-\mu_1^r})$ and 
	$D_{2,i} = z_i\ln(\frac{z_i}{\mu_1^c}) + (1 - z_i)\ln(\frac{1-z_i}{1-\mu_1^c})$.
	
	Based on the above notations and discussions, we can obtain the  following results regarding the right-hand side of \myeqref{eq:core_bound_ts}: $\forall i > 1$ and $k \geq 1$
	\begin{small}
		\begin{align}
			& \myexp\Big\{\frac{1}{p^r_{i , \tau_k + 1}}\Big\} \leq 1 + \Theta\bigg(\frac{3R_{1,i}e^{-D_{1,i}k}}{y_i(1-y_i)(k+1)(R_{1,i}-1)^2} + e^{-2\epsilon^2_ik}\nonumber \\
			&+ \frac{1 + R_{1,i}}{1 - y_i}e^{-D_{1,i}k} + e^{-\frac{1}{2}k\epsilon^2_i} +  \frac{1}{\exp\{\frac{\epsilon^2_ik^2}{2(k+1)}\}-1}\bigg);\label{lemma:bound_reward_for_arm1}
		\end{align}
	\end{small}
	If $z_i \ge 1$, $\myexp\{\frac{1}{p^c_{i,\tau_k+1}}\} = 1$; otherwise,
	\begin{small}
		\begin{align}
			& \myexp\Big\{\frac{1}{p^c_{i,\tau_k+1}}\Big\} \leq 1 + \Theta\bigg(\frac{2e^{-D_{2,i}k}}{z_i(1-z_i)(1-R_{2,i})^2} + e^{-2\epsilon^2_ik} \nonumber\\
			+& \frac{1}{z_iR_{2,i}}e^{-D_{2,i}k} + e^{-\frac{1}{2}\epsilon^2_ik} + \frac{1}{\exp\{\frac{\epsilon_i^2k^2}{2(k+1)}\}-1}\bigg).\label{lemma:bound_cost_for_arm1}
		\end{align}
	\end{small}
	Specifically, if $z_i\geq 1$, $\myexp[\frac{1}{p_{i,\tau_0 + 1}}]\le\frac{1}{1-y_i}$; otherwise
	$\myexp[\frac{1}{p_{i,\tau_0 + 1}}]\le\frac{1}{(1-y_i)z_i}$.
	%
	%
	The derivations of \myeqref{lemma:bound_reward_for_arm1} and \myeqref{lemma:bound_cost_for_arm1} need tight estimations of partial binomial sums and careful algebraic operations, which can be found at Appendix \ref{subapp:boundcomplex1} and \ref{subapp:boundcomplex2}.
	
	According to \myeqref{eq:bound_arm1_hard} and \myeqref{eq:core_bound_ts}, to bound $\sum_{t=1}^{\infty}\myprob\{I_t=i,E_i^\theta(t),B_t>0\}$, we only need to multiply each term in \myeqref{lemma:bound_reward_for_arm1} by each one in \myeqref{lemma:bound_cost_for_arm1}, and sum up all the multiplicative terms over $k$ from $0$ to $\infty$ except the constant $1$.  Using Taylor series expansion, we can verify that w.r.t. $\gamma$, 
	{\small{$$\frac{1}{D_{1,i}} = O\Big(\frac{1}{\epsilon_i^2(\gamma)}\Big), \frac{3R_{1,i}}{y_i(1-y_i)(R_{1,i}-1)^2} = O\Big(\frac{1}{\epsilon_i^2(\gamma)}\Big).$$}}
	If $\epsilon_i(\gamma) + \mu_1^c \ge 1$, we have that w.r.t. $\gamma$,
	\begin{small}
		\begin{equation}
			\sum_{k=0}^{\infty}\Big(\myexp\Big\{\frac{1}{p_{i,\tau_k+1}}\Big\}-1\Big)= O\Big(\frac{1}{\epsilon_i^4(\gamma)}\Big).
			\label{eq:low_prob_event_1}
		\end{equation}
	\end{small}
	If $\epsilon_i(\gamma) + \mu_1^c < 1$, we can obtain that w.r.t $\gamma$,
	\begin{small}
		\begin{equation}
			\sum_{k=0}^{\infty}\Big(\myexp\Big\{\frac{1}{p_{i,\tau_k+1}}\Big\}-1\Big)= O\Big(\frac{1}{\{1-\mu^c_1-\epsilon_i(\gamma)\}\epsilon_i^6(\gamma)}\Big).
			\label{eq:low_prob_event_2}
		\end{equation}
	\end{small}
	Note that the constants in the $O(\cdot)$ of \myeqref{eq:low_prob_event_1} and \myeqref{eq:low_prob_event_2} do not depend on $B$ (but depend on $\mu_i^r$ and $\mu_i^c$ $\forall i \in [K]$).\\
	\emph{\textbf{Step 4}: Bound $\myexp\{T_i\}$ $\forall i\ge2$ for Bernoulli bandits.}
	
	Combining \myeqref{eq:basicDecomposition}, \myeqref{eq:boundthefirsttermAB}, \myeqref{eq:low_prob_event_1} and \myeqref{eq:low_prob_event_2}, we can get the following result:
	\begin{small}
		\begin{align}
			& \myexp\{T_i\}   \leq 1 + \frac{2\ln B}{\delta_i^2(\gamma)} +  \frac{35}{\delta_i^2(\gamma)\mu^c_{\min}} + \Phi_i(\gamma) \nonumber\\
			& \leq 1+\frac{2\ln B}{\gamma^2(\mu_i^c\Delta_i)^2}\bigg(\frac{\mu_1^r}{\mu_1^c}+1\bigg)^2 + O\Big(\frac{1}{\gamma^2}\Big) + \Phi_i(\gamma),	\label{eq:pullingtimeofts_i}
		\end{align}
	\end{small}
	in which $\Delta_i$ is defined in \myeqref{def:Delta} and $\Phi_i(\gamma)$ is defined in \myeqref{eq:BigPhi}.
	
	According to Lemma \ref{lemma:RegretBound_GeneralDistribution}, we can eventually obtain the regret bound  of Budgeted Thompson Sampling as shown in Theorem \ref{thm:bound_TS_asymp} by first multiplying {\small$\mu_i^c\Delta_i$} on the right of \myeqref{eq:pullingtimeofts_i} and then summing over $i$ from $2$ to $K$.
	\subsection{Analysis for General Bandits}\label{subsec:generalExternsion}			
	The regret bound we obtained for Bernoulli bandits in the previous subsection also works for general bandits, as shown in Theorem \ref{thm:bound_TS_asymp}.
	
	The result for general bandits is a little surprising since the problem of general bandits seems more difficult than the Bernoulli bandit problem, and one may expect a slightly looser asymptotic regret bound. The reason why we can retain the same regret bound lies in the Bernoulli trials of the general bandits. Intuitively, the Bernoulli trials can be seen as the intermediate that can transform the general bandits to Bernoulli bandits while keeping the expected reward and cost of each arm unchanged. Therefore, when $B$ is large, there should not be too many differences in the  regret bound between the Bernoulli bandits and general bandits.
	
	Specifically, similar to the case of Bernoulli bandits, in order to bound the regret of the BTS algorithm for the general bandits, we only need to bound $\myexp\{T_i\}$ (according to inequality \myeqref{eq:AGenaralOfTheBerBandit}). To bound $\myexp\{T_i\}$, we also need four steps similar to those described in the previous subsection. In addition, we need one extra step which is related to the Bernoulli trials. Details are described as below.\\
	\emph{$\mathcal{S}0$: Obtain the success probabilities of the Bernoulli trials.} Denote the reward and cost of arm $i$ at round $t$ as $r_i(t)$ and $c_i(t)$ respectively. Denote the Bernoulli trial results of arm $i$ at round $t$ as $\tilde{r}_i(t)$ (for reward) and $\tilde{c}_i(t)$ (for cost). We need to prove $\myprob\{\tilde{r}_i(t)=1\}=\mu_i^r$ and $\myprob\{\tilde{c}_i(t)=1\}=\mu_i^c$, which is straightforward:
	\begin{small}
		\begin{equation*}
			\begin{aligned}
				&\myprob\{\tilde{r}_i(t)=1\} = \myexp\{\myexp[\bm{1}\{\tilde{r}_i(t)=1\} | r_i(t)]\} = \myexp[r_i(t)] = \mu_i^r,\\
				&\myprob\{\tilde{c}_i(t)=1\} = \myexp\{\myexp[\bm{1}\{\tilde{c}_i(t)=1\} | c_i(t)]\} = \myexp[c_i(t)] = \mu_i^c.
			\end{aligned}
		\end{equation*}
	\end{small}
	\emph{$\mathcal{S}1$: Decompose $\myexp\{T_i\}$}: This step is the same as Step 1 in the Bernoulli bandit case. For the general bandit case, $\myexp\{T_i\}$ can also be bounded by inequality \myeqref{eq:basicDecomposition}.\\
	\emph{$\mathcal{S}2$: Bound {$\sum_{t=1}^{\infty}\myprob\{\overline{E_i^\theta(t)},n_{i,t}\ge \lceil L_i\rceil, B_t > 0 \}$}.} $\mathcal{S}2$ is almost the same as Step 2 in the proof for Bernoulli bandits but contains some minor changes.  For the general bandits, we have $c_i(t)\in[0,1]$ rather than $c_i(t)\in\{0,1\}$. Then we have {$\sum_{t=1}^{\infty}\myprob\{B_t > 0\} \leq \frac{B+1}{\mu^c_{\min}}$}, and can get a similar result to \myeqref{eq:boundthefirsttermAB}.\\
	\emph{$\mathcal{S}3$: Bound $\sum_{t=1}^{\infty}\myprob\{I_t=i,E_i^\theta(t),B_t > 0\}$.} Since we have already got the success probabilities of the Bernoulli trials, this step is the same as Step 3 for the Bernoulli bandits.\\
	\emph{$\mathcal{S}4$: }Substituting the results of $\mathcal{S}2$ and $\mathcal{S}3$ into the corresponding terms in \myeqref{eq:basicDecomposition}, we can get an upper bound of $\myexp\{T_i\}$ for the general bandits. Then according to \myeqref{eq:AGenaralOfTheBerBandit}, for general bandits, the results in Theorem \ref{thm:bound_TS_asymp} can be eventually obtained.
	
	The classical MAB problem in \cite{UCB1} can be regarded as a special case of the budgeted MAB problem by setting $c_i(t)=1$ $\forall i\in[K],t\ge1$, and $B$ is the maximum pulling time. Therefore, according to \cite{lai1985asymptotically}, we can verify the order of the distribution-dependent regret bound of the budgeted MAB problem is $O(\ln B)$. Compared with the two algorithms in \cite{ding2013multi}, we have the following results:
	\begin{remark}
		By setting $\gamma=\frac{1}{\sqrt{2}}$ in Theorem \ref{thm:bound_TS_asymp}, we can see that BTS gets a tighter asymptotic regret bound in terms of the constants before $\ln B$ than the two algorithms proposed in \cite{ding2013multi}.
		\label{remark:comparedwithwenkui}
	\end{remark}

\iftrue
\section{Numerical Simulations}

\begin{figure*}[!ht]
	\centering
	\begin{minipage}[t]{0.5\linewidth}
		\subfigure[Bernoulli, $10$ arms]{
			\includegraphics[width = \linewidth]{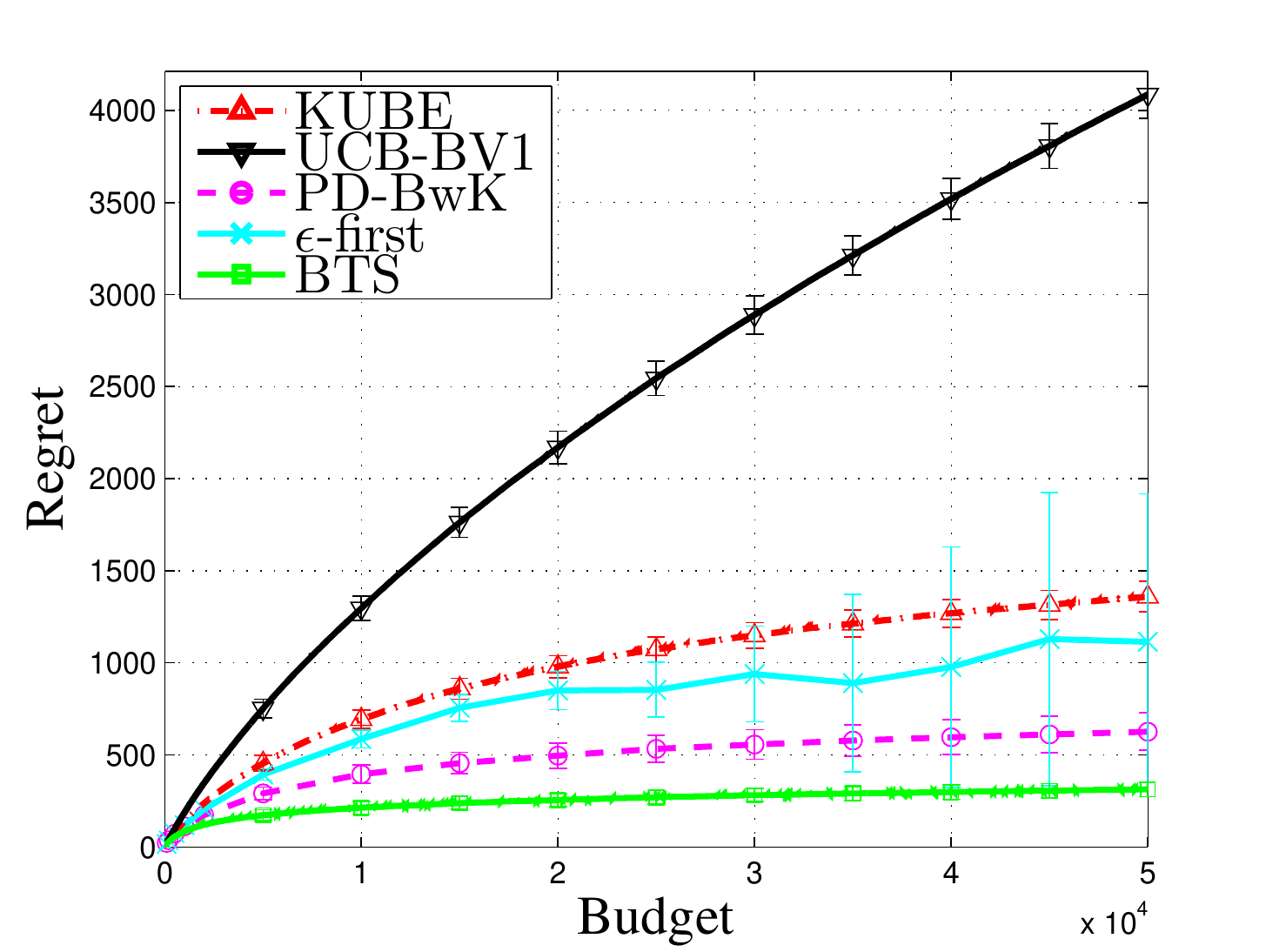}
			\label{subfig_10bern}
		}
	\end{minipage}%
	\begin{minipage}[t]{0.5\linewidth}
		\subfigure[Bernoulli, $100$ arms]{
			\includegraphics[width = \linewidth]{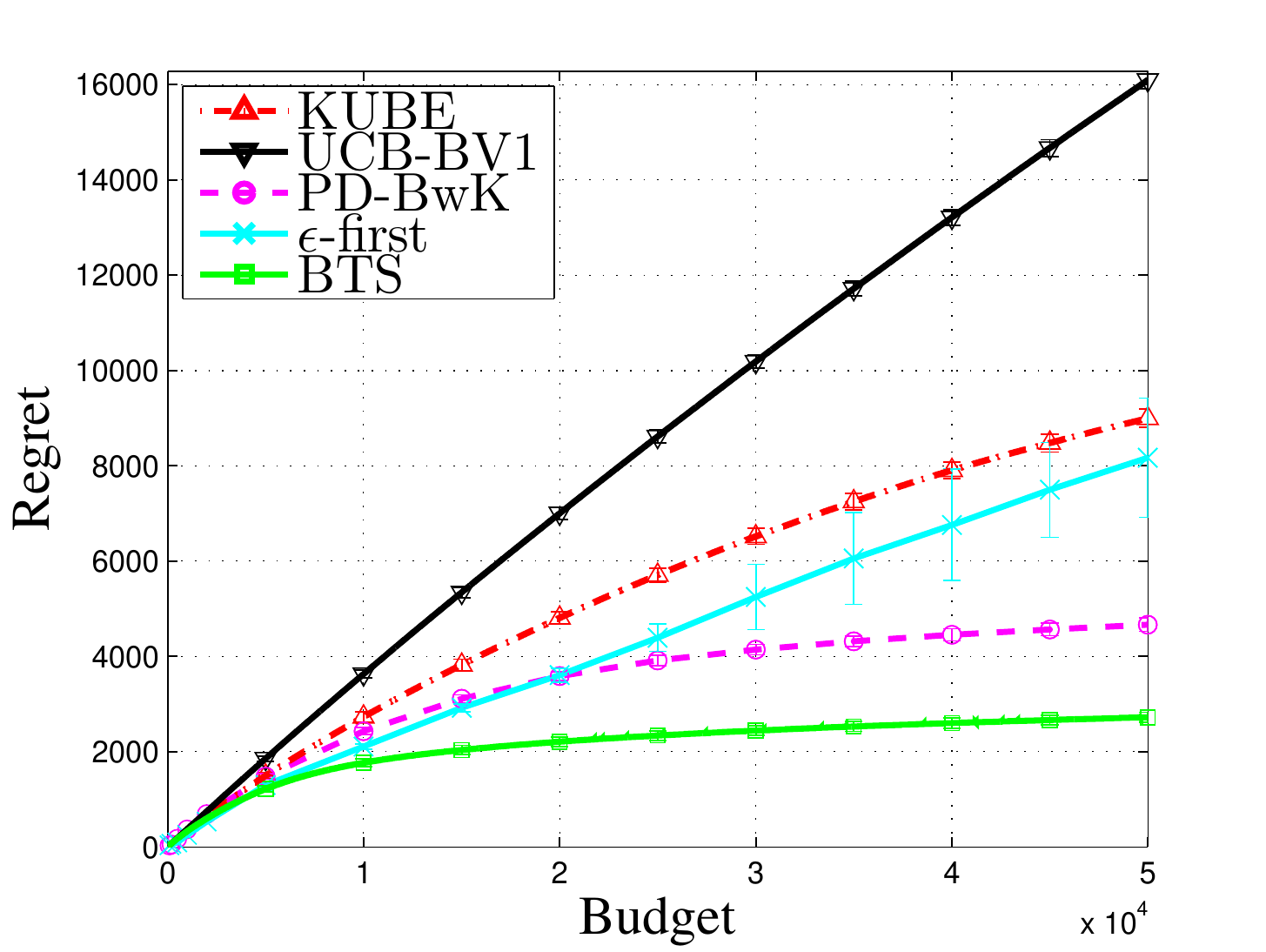}
			\label{subfig_100bern}
		}
	\end{minipage}
	\begin{minipage}[t]{0.5\linewidth}
		\subfigure[Multinomial, $10$ arms]{
			\includegraphics[width = \linewidth]{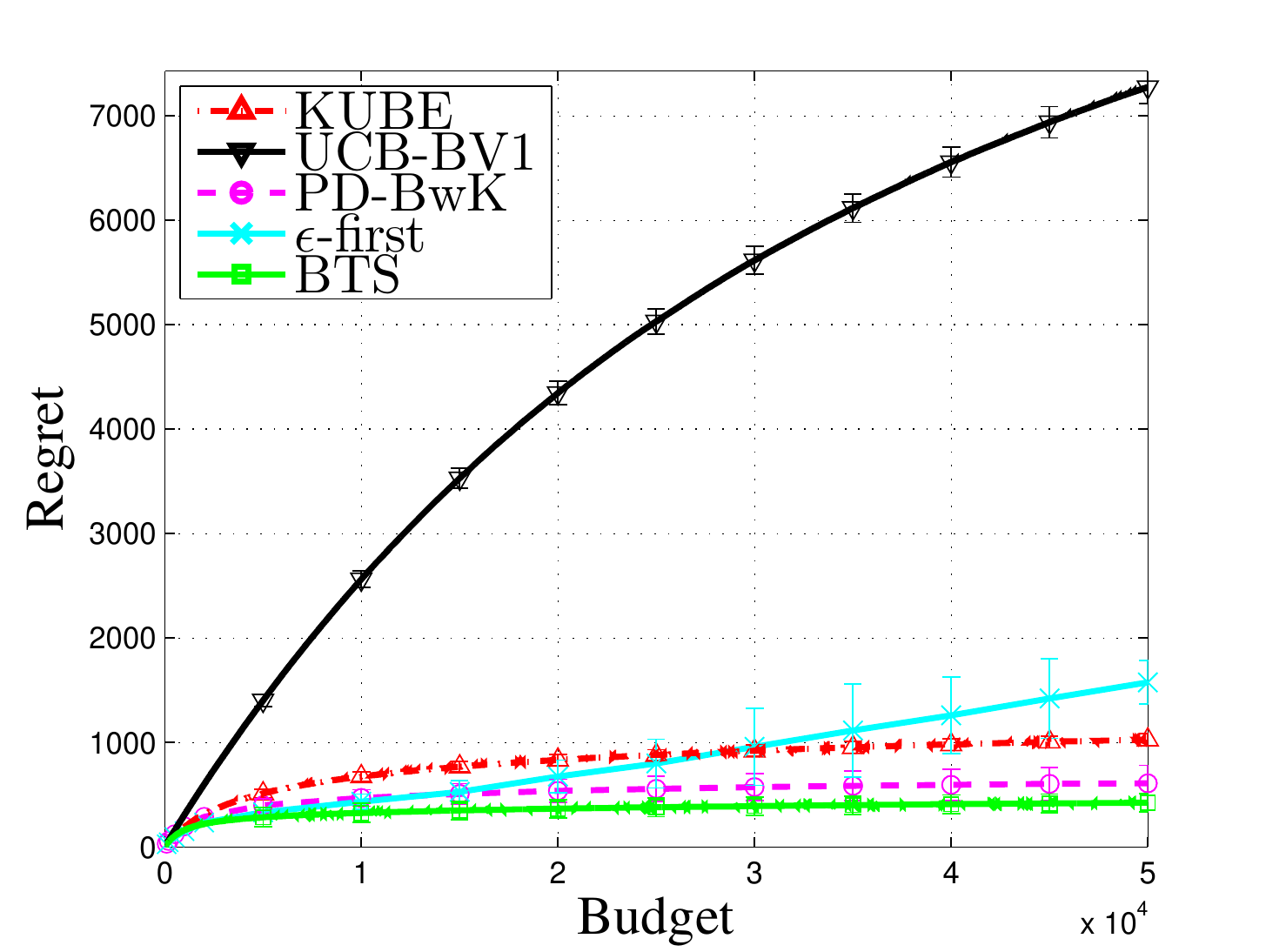}
			\label{subfig_10multi}
		}
	\end{minipage}%
	\begin{minipage}[t]{0.5\linewidth}
		\subfigure[Multinomial, $100$ arms]{
			\includegraphics[width = \linewidth]{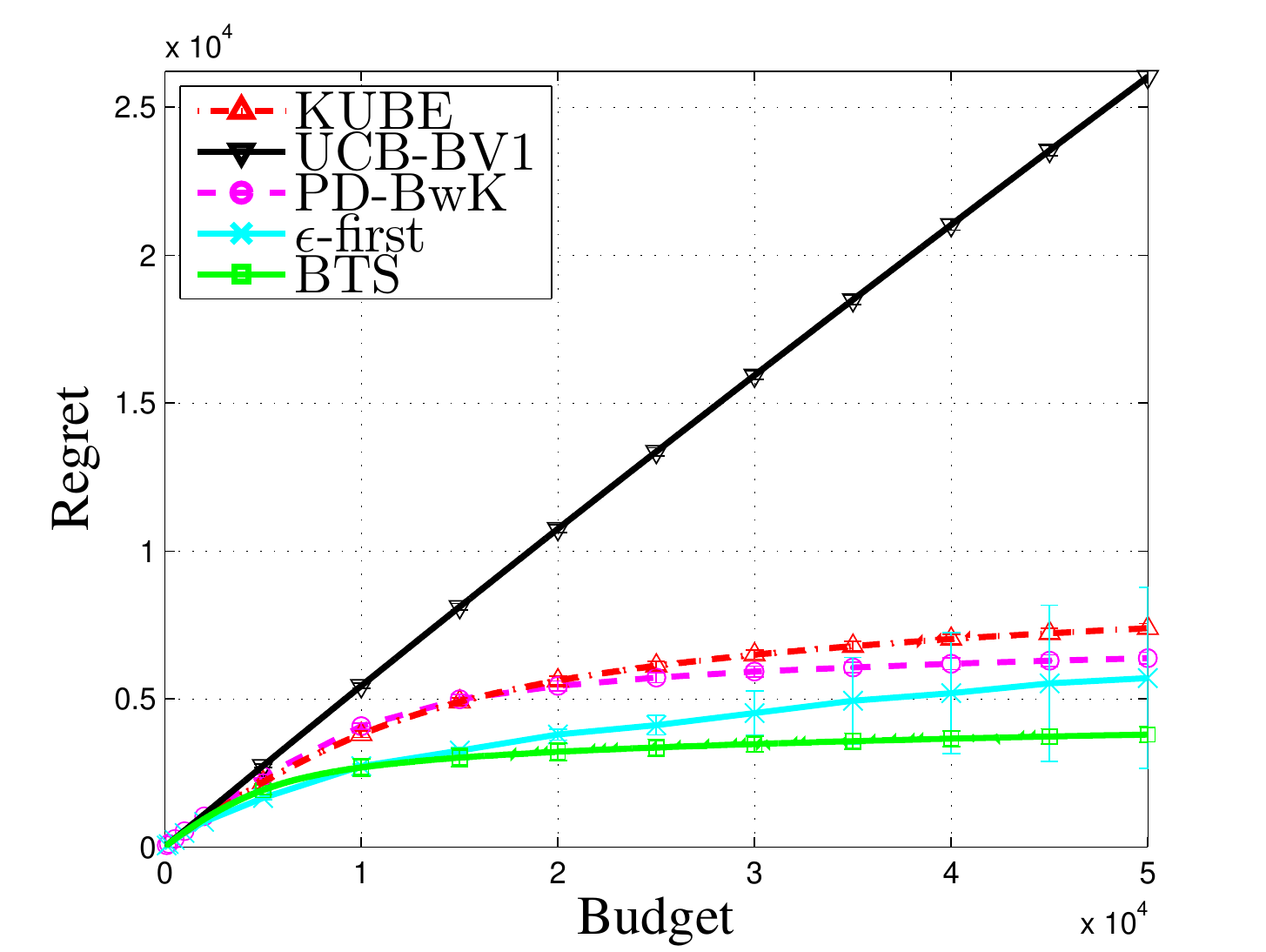}
			\label{subfig_100multi}
		}
	\end{minipage}
	\caption{Regrets under different bandit settings}
	\label{fig:regrets}
\end{figure*}
In addition to the theoretical analysis of the BTS algorithm, we are also interested in its empirical performance. We conduct a set of experiments to test the empirical performance of BTS algorithm and present the results in this section.

For comparison purpose, we implement four baseline algorithms: (1) the $\epsilon$-first algorithm \cite{epsilon_first} with $\epsilon=0.1$; (2) a variant of the PD-BwK algorithm \cite{badanidiyuru2013bandits}: at each round, pull the arm with the maximum $\frac{\min\{\overline{r}_{i,t} + \varphi(\overline{r}_{i,t},n_{i,t}),1\}}{\max\{\overline{c}_{i,t} - \varphi(\overline{c}_{i,t},n_{i,t}),0\}}$, in which $\overline{r}_{i,t}$ ($\overline{c}_{i,t}$) is the average reward (cost) of arm $i$ before round $t$, $\varphi(x,N)=\sqrt{\frac{\nu x}{N}} + \frac{\nu}{N}$ and $\nu=0.25\log(BK)$; (3) the UCB-BV1 algorithm \cite{ding2013multi}; (4) a variant of the KUBE algorithm \cite{KUBE}: at round $t$, pull the arm with the maximum ratio $\big(\overline{r_{i,t}} + \sqrt{\frac{2\ln t}{n_{i,t}}}\big)\big/\overline{c_{i,t}}$.
$\epsilon$-first and PD-BwK need to know $B$ in advance, and thus we try several budgets as $\{100,200,500,1K,2K,5K,10K,15K,20K,\cdots,50K\}$. BTS and UCB-BV1 do not need to know $B$ in advance, and thus by setting $B=50K$ we can get their empirical regrets for every budget smaller than $50K$.

We simulate bandits with two different distributions: one is Bernoulli distribution (simple), and the other is multinomial distribution (complex). Their parameters are randomly chosen. For each distribution, we simulate a $10$-armed case and a $100$-armed case. We then independently run the experiments for $500$ times and report the average performance of each algorithm.

The average regret and the standard deviation of each algorithm over 500 random runs are shown in Figure \ref{fig:regrets}.  From the figure we have the following observations:
\begin{itemize}
	\item For both the Bernoulli distribution and the multinomial distribution, and for both the 10-arm case and 100-arm case, our proposed BTS algorithm has clear advantage over the baseline methods: It achieves the lowest regrets. Furthermore, the standard deviation of the regrets of BTS over 500 runs is  small, indicating that its performance is very stable across different random run of the experiments.
	\item As the number of arms increases (from 10 to 100), the regrets of all the algorithms increase, given the same budget. This is easy to understand because more budget is required to make good explorations on more arms.
	\item 	The standard deviation of the regrets of the $\epsilon$-first algorithm is much larger than the other algorithms, which shows that $\epsilon$-first is not stable under certain circumstances. Take the $10$-armed Bernoulli bandit for example: when $B=50K$, during the $500$ random runs, there are $13$ runs that $\epsilon$-first cannot identify the optimal arm. The average regret over the $13$ runs is $4630$. However, over the other $487$ runs, the average regret of $\epsilon$-first is $1019.9$. Therefore, the standard derivation of $\epsilon$-first is large. In comparison, the BTS algorithm is much more stable.
\end{itemize}

Overall speaking, the simulation results demonstrate the effectiveness of our proposed Budgeted Thompson Sampling algorithm.

\fi

\section{Conclusion and Future work}
In this paper, we have extended the Thompson sampling algorithm to the budgeted MAB problems. We have proved that our proposed algorithm has a distribution-dependent regret bound of $O(\ln B)$. We have also demonstrated its empirical effectiveness using several numerical simulations.

For future work, we plan to investigate the following aspects: (1) We will study the distribution-free regret bound of Budgeted Thompson Sampling. (2) We will try other priors (e.g., the Gaussian prior) to see whether a better regret bound and empirical performance can be achieved in this way. (3) We will study the setting that the reward and the cost are correlated (e.g., an arm with higher reward is very likely to have higher cost).

\bibliographystyle{named}
\bibliography{mybib}

\clearpage
\onecolumn

\appendix
\section{Appendix: Some Important Facts}
\begin{fact}[Chernoff-Hoeffding Bound, \cite{UCB1}]
Let $X_1,\cdots,X_n$ be random variables with common range $[0,1]$ and such that $\myexp[X_t|X_1,\cdots,X_{t-1}]=\mu$. Let $S_n=X_1+\cdots+X_n$. Then for all $a \geq 0$,
\begin{small}
\begin{equation}
\myprob\{S_n \ge n\mu + a\} \leq e^{-\frac{2a^2}{n}};\;\myprob\{S_n \le n\mu - a\} \leq e^{-\frac{2a^2}{n}}.
\end{equation}
\end{small}
\label{fact:hoeffdinginequality}
\end{fact}
Throughout the appendices, let $F^{Beta}_{\alpha,\beta}(\cdot)$ denote the cdf of a beta distribution with parameters $\alpha$ and $\beta$. (In our analysis, $\alpha$ and $\beta$ are two integers.) Let $F^B_{n,p}(\cdot)$ denote the cdf the binomial distribution, in which $n (\in \mathbb{Z}_+)$ is the number of the Bernoulli trials and $p$ is  the success probability of each trial.
\begin{fact}
For any positive integer $\alpha$ and $\beta$,
\begin{equation}
  F_{\alpha,\beta}^{Beta}(y) = 1 - F_{\alpha+\beta-1 , y}^B(\alpha-1).
\end{equation}
\label{fact:beta_binomial}
\end{fact}
\begin{proof}
\begin{small}
\begin{equation*}
\begin{aligned}
 F_{\alpha,\beta}^{Beta}(y) & =  \int_{0}^{y}\frac{(\alpha + \beta - 1)!}{(\alpha - 1)!(\beta - 1)!}t^{\alpha - 1}(1 - t)^{\beta - 1}\mathrm{d}t
= \frac{(\alpha + \beta - 1)!}{\alpha!(\beta-1)!}y^{\alpha}(1-y)^{\beta - 1} + \int_{0}^{y}\frac{(\alpha+\beta-1)!}{\alpha!(\beta-2)!} t^\alpha(1-t)^{\beta-2}\mathrm{d}t \\
& = \cdots = \frac{(\alpha + \beta - 1)!}{\alpha!(\beta-1)!}y^{\alpha}(1-y)^{\beta - 1} + \frac{(\alpha + \beta - 1)!}{(\alpha + 1)!(\beta-2)!}y^{\alpha+1}(1-y)^{\beta - 2} + \cdots + y^{\alpha + \beta - 1} \\
& = \sum_{k=0}^{\beta-1}\binom{\alpha+\beta-1}{\alpha+k}y^{\alpha + k}(1-y)^{\beta - 1 - k} = 1 - F^B_{\alpha+\beta-1,y}(\alpha-1).
\end{aligned}
\end{equation*}
\end{small}
\end{proof}
\begin{fact}
  \begin{equation}
  F_{n+1,p}^B(r) = (1-p)F_{n,p}^B(r) + pF_{n,p}^B(r-1) \leq (1-p)F_{n,p}^B(r) + pF_{n,p}^B(r) \leq F_{n,p}^B(r).
  \end{equation}
  \label{fact:Binomial_inequality}
\end{fact}
\begin{fact}
For all $p\in[0,1]$, $\delta > 0, n \in \mathbb{Z}_+$,
  \begin{equation}
  \begin{aligned}
  & F_{n,p}^B(np - n\delta) \leq e^{-2n\delta^2}, \\
  & 1 - F_{n , p}^B(np+n\delta) \leq e^{-2n\delta^2}. \\
  \end{aligned}
  \end{equation}
For all $p\in[0,1]$, $\delta > 0, n \in \mathbb{Z}_+$ and $n > \frac{1}{\delta}$,
 \begin{equation}
1 - F_{n+1,p}^B(np + n\delta) \leq e^{4\delta-2n\delta^2}.
 \end{equation}
\label{fact:Binomial_hoeff}
\end{fact}
\begin{proof}
$\forall t\ge1$, let $X_t$ denote the result for the $t$-th Bernoulli trial, whose success probability $p$. $\{X_t\}_{t=1}^{n}$ are independent and identically distributed.
\begin{equation}
\begin{aligned}
& F_{n,p}^B(np-n\delta) = \myprob\bigg\{\sum_{t=1}^{n}X_t \leq np - n\delta\bigg\} = \myprob\bigg\{\sum_{t=1}^{n}X_t  - \myexp\bigg[\sum_{t=1}^{n}X_t\bigg] \leq - n\delta\bigg\} \leq e^{-2n\delta^2};\\
& 1 - F_{n,p}^B(np + n\delta) \leq \myprob\bigg\{\sum_{t=1}^{n}X_t \geq np + n\delta\bigg\} \leq \myprob\bigg\{\sum_{t=1}^{n}X_t - \myexp\bigg[\sum_{t=1}^{n}X_t\bigg]\geq  n\delta\bigg\} \leq e^{-2n\delta^2}.
\end{aligned}
\end{equation}
For the third term, we first declare that
\begin{equation}
F_{n+1,p}^B(np + n\delta) = (1 - p)F_{n,p}^B(np+n\delta) + pF_{n,p}^B(np + n\delta - 1) \ge F_{n,p}^B(np + n\delta - 1) .
\end{equation}
As a result,
\begin{equation}
1 - F_{n+1,p}^B(np + n\delta) \le 1 - F_{n,p}^B(np + n\delta - 1) \le e^{-2n(\delta - \frac{1}{n})^2} \le e^{4\delta - 2n\delta^2}.
\end{equation}
\end{proof}

\begin{fact}[Section B.3 of \cite{export:191857}]
For Binomial distribution,
\begin{enumerate}
\item If $s \leq y(j+1)-\sqrt{(j+1)y(1-y)}$, $F^B_{j+1,y}(s)=\Theta\left(\frac{y(j+1-s)}{y(j+1)-s}\binom{j+1}{s}y^s(1-y)^{j+1-s}\right)$;
\item If $s \geq y(j+1)-\sqrt{(j+1)y(1-y)}$, $F^B_{j+1,y}(s)=\Theta\left(1\right)$.
\end{enumerate}
Similarly, we can obtain
\begin{enumerate}
\item If $j-s\leq(1-y)(j+1)-\sqrt{(j+1)y(1-y)}$, $1 - F^B_{j+1,y}(s) = \Theta\left(\frac{(1-y)(s+1)}{(1-y)(j+1)-j+s}\binom{j+1}{j-s}(1-y)^{j-s}y^{s+1}\right)$;
\item If $j-s\geq(1-y)(j+1)-\sqrt{(j+1)y(1-y)}$, $1 - F^B_{j+1,y}(s) = \Theta\left(1\right)$.
\end{enumerate}
\end{fact}
We give a proof of the latter two cases:
\begin{equation}
\begin{aligned}
& F^B_{j+1,y}(s) = \sum_{k=0}^{s}\binom{j+1}{k}y^k(1-y)^{j+1-k} \\
& 1 - F^B_{j+1,y}(s) = \sum_{k=s+1}^{j+1}\binom{j+1}{k}y^k(1-y)^{j+1-k} = \sum_{k=0}^{j-s}\binom{j+1}{k}(1-y)^ky^{j+1-k}.
\end{aligned}
\end{equation}
Therefore, in the original conclusion, by replacing the $s$ with $j-s$ and $y$ with $1-y$, we can get the latter two equations.

\section{Appendix: Omitted Proofs}
\subsection{Proof of Lemma \ref{lemma:optimalPolicy2}}\label{app:proof_lemma_opt}
\emph{We first prove the result for the Bernoulli bandits.} \\
Denote $R^*(b)$ as the expected optimal revenue when the left budget is $b$ ($b$ is a non-negative integer). Define $R^*(0) = 0$. Assume the optimal policy is to pull arm $i\in[K]$ when the remaining budget is $b$. We have
\begin{small}
	\begin{equation}
		\begin{aligned}
			R^*(b) = (1 - \mu_i^c)(1 - \mu_i^r)R^*(b) + (1 - \mu_i^c)\mu_i^r(1 + R^*(b)) + \mu_i^c(1 - \mu_i^r)R^*(b - 1) + \mu_i^c\mu_i^r(1 + R^*(b - 1)).
		\end{aligned}
		\label{eq:optRev1}
	\end{equation}
\end{small}
After some derivations we can get
\begin{small}
	\begin{equation}
		R^*(b) = R^*(b - 1) + \frac{\mu_i^r}{\mu_i^c} \leq R^*(b - 1) + \frac{\mu_1^r}{\mu_1^c} \leq b\frac{\mu_1^r}{\mu_1^c}.
	\end{equation}
\end{small}
Since we set that $B$ is a positive integer, we have $R^*(B) \leq B\frac{\mu_1^r}{\mu_1^c}$. On the other hand, if we always pull arm $1$, with the similar derivation of \myeqref{eq:optRev1}, we can obtain the expected reward is just $B\frac{\mu_1^r}{\mu_1^c}$. Therefore always pulling arm $1$ is the optimal policy for Bernoulli bandits. \\
\emph{Next we prove the result for the general bandits.}\\
Let $B_t$ denote the remaining budget before (excluding) time $t$, and $r_k(t)$ ($c_k(t)$) denote the reward (cost) of arm $k$ at round $t$. Please note that $r_k(t)$ and $c_k(t)$ will always exist $\forall k \in [K], t \ge 1$. Only if arm $k$ is pulled, the reward $r_k(t)$ and the cost $c_k(t)$ will be given to the player. For any algorithm, the expected reward $\overline{\textrm{REW}}$ can be upper bounded by
\begin{small}
	\begin{equation}
		\begin{aligned}
			\overline{\textrm{REW}} \leq^* &\myexp\sum_{k=1}^{K}\sum_{t=1}^{\infty}r_k(t)\bm{1}\{I_t=k,B_t > 0\}
			=\sum_{k=1}^{K}\sum_{t=1}^{\infty}\myexp [r_k(t)\bm{1}\{I_t=k,B_t > 0\}]\\
			=&\sum_{k=1}^{K}\sum_{t=1}^{\infty}\mu_k^r\myprob\{I_t=k,B_t > 0\}
			=\sum_{k=1}^{K}\sum_{t=1}^{\infty}\frac{\mu_k^r}{\mu_k^c}\mu_k^c\myprob\{I_t=k,B_t > 0\}\\
			=&\sum_{k=1}^{K}\sum_{t=1}^{\infty}\frac{\mu_k^r}{\mu_k^c}\myexp [c_k(t)\bm{1}\{I_t=k,B_t > 0\}]
			\leq\sum_{k=1}^{K}\sum_{t=1}^{\infty}\frac{\mu_1^r}{\mu_1^c}\myexp [c_k(t)\bm{1}\{I_t=k,B_t > 0\}]\\
			=&\frac{\mu_1^r}{\mu_1^c}\myexp \sum_{k=1}^{K}\sum_{t=1}^{\infty}[c_k(t)\bm{1}\{I_t=k,B_t > 0\}]  \leq^{\triangle} \frac{\mu_1^r}{\mu_1^c}(B + 1).
		\end{aligned}
	\end{equation}
\end{small}
The inequality with superscript $^*$ holds because if $B_{t}>0$ but $B_{t+1}<0$, the player cannot get the reward at round $t$ and the game stops.
The inequality with superscript $^\triangle$ holds because $\sum_{k=1}^{K}\sum_{t=1}^{\infty}[c_k(t)\bm{1}\{I_t=k,B_t > 0\}]$ is the total cost of the pulled arms before the budget runs out. For general bandits, it is probable that $c_k(t) > B_t$. As a result, $\sum_{k=1}^{K}\sum_{t=1}^{\infty}[c_k(t)\bm{1}\{I_t=k,B_t > 0\}] \le B+1$.
Therefore, for the general bandits, we have $R^* \le \frac{\mu_1^r}{\mu_1^c}(B + 1)$.

If the player keeps pulling arm $1$, the expected reward $\underline{\textrm{REW}}$ is at least:
\begin{small}
	\begin{equation}
		\begin{aligned}
			\underline{\textrm{REW}} \geq &\myexp\sum_{t=1}^{\infty}r_1(t)\bm{1}\{I_t=1,B_t \geq 1\}
			=\sum_{t=1}^{\infty}\myexp[ r_1(t)\bm{1}\{I_t=1,B_t \geq 1\}]\\
			=& 
			\sum_{t=1}^{\infty}\mu_1^r\myprob\{I_t=1,B_t \geq 1\}
			=\sum_{t=1}^{\infty}\frac{\mu_1^r}{\mu_1^c}\mu_1^c\myprob\{I_t=1,B_t \geq 1\}\\
			=&\sum_{t=1}^{\infty}\frac{\mu_1^r}{\mu_1^c}\myexp [c_1(t)\bm{1}\{I_t=1,B_t \ge 1\}]
			\ge\frac{\mu_1^r}{\mu_1^c}(B-1).
		\end{aligned}
	\end{equation}
\end{small}
Therefore, the sub-optimality of always pulling arm $1$ compared to the optimal policy is at most $\frac{2\mu^r_1}{\mu^c_1}$.

\subsection{Proof of Eqn. \myeqref{eq:RegretBound_BernoulliDistibutrion} in Lemma \ref{lemma:RegretBound_GeneralDistribution}}\label{subsec:proof_pulltoregret}
First, we will find an equivalent expression of the expected reward (denoted as REW). Still, let $B_t$ denote the remaining budget before (excluding) round $t$, and $r_k(t)$ ($c_k(t)$) denote the reward (cost) of arm $k$ at round $t$.
In addition,  $B^{(k)}$ denotes the budget spent by arm $k$ when the algorithm stops.
\begin{small}
	\begin{equation}
		\begin{aligned}
			\textrm{REW} =&\myexp\sum_{k=1}^{K}\sum_{t=1}^{\infty}r_k(t)\bm{1}\{I_t=k,B_t > 0\}
			=\sum_{k=1}^{K}\sum_{t=1}^{\infty}\myexp\{r_k(t)\bm{1}\{I_t=k,B_t > 0\}\}\\
			=&\sum_{k=1}^{K}\sum_{t=1}^{\infty}\mu_k^r\myprob\{I_t=k,B_t > 0\}
			=\sum_{k=1}^{K}\sum_{t=1}^{\infty}\frac{\mu_k^r}{\mu_k^c}\mu_k^c\myprob\{I_t=k,B_t > 0\}\\
			=&\sum_{k=1}^{K}\sum_{t=1}^{\infty}\frac{\mu_k^r}{\mu_k^c}\myexp\{c_k(t) \bm{1}\{I_t=k,B_t > 0\}\}\\
			=&\sum_{k=1}^{K}\frac{\mu_k^r}{\mu_k^c}\myexp\sum_{t=1}^{\infty} c_k(t)\bm{1}\{I_t=k,B_t > 0\}=\sum_{k=1}^{K}\frac{\mu_k^r}{\mu_k^c}\myexp B^{(k)}.\\
		\end{aligned}
	\end{equation}
\end{small}
According to our assumption, we know the algorithm will stop when the budget runs out. We have already set that $B$ is an integer, and the cost of Bernoulli bandits per pulling is either $0$ or $1$. Thus, we know that when the algorithm stops, the budget exactly runs out. That is, $\sum_{k=1}^{K}B^{(k)} = B.$

The optimal reward for Bernoulli bandit is $\frac{\mu_1^r}{\mu_1^c}B$, which is given in Lemma \ref{lemma:optimalPolicy2}. Thus, the regret can be written as
\begin{small}
	\begin{equation}
		\textrm{Regret}  = \frac{\mu_1^r}{\mu_1^c}B - \sum_{k=1}^{K}(\frac{\mu_k^r}{\mu_k^c})\myexp B^{(k)} =\sum_{k=2}^{K}(\frac{\mu_1^r}{\mu_1^c}-\frac{\mu_k^r}{\mu_k^c})\myexp B^{(k)}=
		\sum_{k=2}^{K}\Delta_k\myexp B^{(k)}
	\end{equation}
\end{small}
And we can verify that
\begin{small}
	\begin{equation}
		\begin{aligned}
			\myexp B^{(k)} &= \myexp\sum_{t=1}^{\infty}c_k(t)\bm{1}\{I_t=k, B_t > 0 \}\\
			&=\sum_{t=1}^{\infty}\mu_k^c\myprob\{I_t=k,, B_t > 0 \} = \mu_k^c\sum_{t=1}^{\infty}\myexp\{I_t=k,, B_t > 0 \} = \mu_k^c\myexp[T_k].
		\end{aligned}
	\end{equation}
\end{small}
Therefore, the regret could be written as $\textrm{Regret} =\sum_{k=2}^{K}\Delta_k\mu_k^c\myexp T_{k}.$\\

\subsection{Proof of inequality \myeqref{eq:AGenaralOfTheBerBandit} in Lemma \ref{lemma:RegretBound_GeneralDistribution}}\label{app:proof_pull_regret}
For any policy, we can obtain the expected reward REW is at least
\begin{small}
	\begin{equation}
		\begin{aligned}
			\textrm{REW} \geq &\myexp\sum_{k=1}^{K}\sum_{t=1}^{\infty}r_k(t)\bm{1}\{I_t=k,B_t \geq 1\}
			=\sum_{k=1}^{K}\sum_{t=1}^{\infty}\myexp[ r_k(t)\bm{1}\{I_t=k,B_t \geq 1\}]
			=\sum_{k=1}^{K}\sum_{t=1}^{\infty}\mu_k^r\myprob\{I_t=k,B_t \geq 1\}\\
			=&\sum_{k=1}^{K}\sum_{t=1}^{\infty}\frac{\mu_k^r}{\mu_k^c}\mu_k^c\myprob\{I_t=k,B_t \geq 1\}
			=\sum_{k=1}^{K}\sum_{t=1}^{\infty}\frac{\mu_k^r}{\mu_k^c}\myexp [c_k(t)\bm{1}\{I_t=k,B_t \ge 1\}]\\
			=&\sum_{k=1}^{K}\frac{\mu_k^r}{\mu_k^c}\myexp\sum_{t=1}^{\infty} c_k(t)\bm{1}\{I_t=k,B_t \ge 1\}
			\overset{\textrm{def}}{=}\sum_{k=1}^{K}\frac{\mu_k^r}{\mu_k^c}\myexp \tilde{B}^{(k)},
		\end{aligned}
	\end{equation}
\end{small}
One can verify that $\sum_{k=1}^{K}\tilde{B}^{(k)} \ge B - 1$. As a result, the regret can be written as
\begin{equation}
	\begin{aligned}
		\textrm{Regret}  \leq 2\frac{\mu_1^r}{\mu_1^c} + \frac{\mu_1^r}{\mu_1^c}\sum_{k=1}^{K}\myexp \tilde{B}^{(k)} - \sum_{k=1}^{K}\frac{\mu_k^r}{\mu_k^c}\myexp \tilde{B}^{(k)}
		\leq2\frac{\mu_1^r}{\mu_1^c}+\sum_{k=2}^{K}\Delta_k\myexp[\tilde{B}^{(k)}] \leq 2\frac{\mu_1^r}{\mu_1^c} + \sum_{k=2}^{K}\Delta_k\myexp[B^{(k)}].
	\end{aligned}
\end{equation}
Again, re-write $\myexp[B^{(k)}]$ using the indicator function:
\begin{small}
	\begin{equation}
		\begin{aligned}
			\myexp[B^{(k)}] \leq &\myexp\sum_{t=1}^{\infty}c_k(t)\bm{1}\{I_t=k,B_t > 0\}
			=\sum_{t=1}^{\infty}\myexp[c_k(t)\{I_t=k,B_t > 0\}]\\
			=&\sum_{t=1}^{\infty}\mu_k^c\myexp[\bm{1}\{I_t=k,B_t > 0\}]
			=\mu_k^c\myexp\sum_{t=1}^{\infty}\bm{1}\{I_t=k,B_t > 0\}
			\leq\mu_k^c\myexp[T_k],
		\end{aligned}
	\end{equation}
\end{small}
Therefore,
\begin{small}
	\begin{equation}
		\textrm{Regret} \leq2\frac{\mu_1^r}{\mu_1^c}+\sum_{k=2}^{K}\Delta_k\myexp[B^{(k)}]\leq2\frac{\mu_1^r}{\mu_1^c}+\sum_{k=2}^{K}\Delta_k\mu_k^c\myexp[T_k].
	\end{equation}
\end{small}

\subsection{Derivation of inequality \myeqref{eq:basicDecomposition}}\label{app:coredecomposition}
\begin{small}
\begin{align}
&\myexp\{T_i\} = \myexp\big\{\sum_{t=1}^{\infty}\bm{1}\{I_t = i, B_t > 0\}\big\}=  \myexp\big\{\sum_{t=1}^{\infty}\bm{1}\{I_t = i, B_t > 0,\overline{E_i^\theta(t)}\}\big\}+ \myexp\big\{\sum_{t=1}^{\infty}\bm{1}\{I_t = i,B_t>0,E_i^\theta(t)\}\big\}\nonumber\\
\leq & \lceil L_i \rceil + \myexp\big\{\sum_{t=1}^{\infty}\bm{1}\{I_t = i, B_t > 0,\overline{E_i^\theta(t)},n_{i,t}\ge \lceil L_i \rceil \}\big\}
	+ \sum_{t=1}^{\infty}\myprob\{I_t = i,B_t>0,E_i^\theta(t)\}\nonumber\\
\leq & \lceil L_i \rceil + \sum_{t=1}^{\infty}\myprob\{\overline{E_i^\theta(t)},n_{i,t}\ge \lceil L_i \rceil, B_t > 0 \}
	+ \sum_{t=1}^{\infty}\myprob\{I_t = i,E_i^\theta(t),B_t>0\}.
\end{align}
\end{small}

\subsection{Derivation of inequality \myeqref{eq:bound_firstEqn}} \label{app:proof_low:AA}
Define $A_i^r(t)$ as the event: $  A_i^r(t):\;\frac{S_i^r(t)}{n_{i,t}} \leq \mu_i^r + \frac{\delta_i(\gamma)}{2}.$
We know that
\begin{align}
& \myprob\{ \overline{E^r_i(t)} , n_{i,t} \geq \lceil L_i\rceil|B_t>0\} = \myprob\{\theta_i^r(t) > \mu_i^r + \delta_i(\gamma),n_{i,t} \geq \lceil L_i\rceil|B_t>0\} \nonumber\\
=  & \myprob\{\theta_i^r(t) > \mu_i^r + \delta_i(\gamma),A_i^r(t),n_{i,t} \geq \lceil L_i\rceil|B_t>0\}  + \myprob\{\theta_i^r(t) > \mu_i^r + \delta_i(\gamma), \overline{A_i^r(t)},n_{i,t} \geq \lceil L_i\rceil|B_t>0\} \nonumber\\
\leq & \myprob\{\overline{A_i^r(t)},n_{i,t} \geq \lceil L_i\rceil|B_t>0\} + \myprob\{\theta_i^r(t) > \mu_i^r + \delta_i(\gamma);A_i^r(t),n_{i,t} \geq \lceil L_i\rceil|B_t>0\} . \label{eq:prAY2}
\end{align}
For the first term of \myeqref{eq:prAY2}:
\begin{small}
	\begin{equation}
	\begin{aligned}
	& \myprob\{\overline{A_i^r(t)} , n_{i,t} \geq \lceil L_i\rceil|B_t>0\} \leq\sum_{l = \lceil L_i\rceil}^{\infty}\myprob\{\overline{A_i^r(t)} , n_{i,t} =l|B_t>0\}\\
	\leq & \sum_{l=\lceil L_i \rceil}^{\infty}\myprob\{ S_i^r(t) - n_{i,t}\mu_i^r > n_{i,t}\frac{\delta_i(\gamma)}{2} | n_{i,t} = l,B_t>0\} \\
	= & \sum_{l=\lceil L_i \rceil}^{\infty}\myprob\{ S_i^r(t) - l\mu_i^r > l\frac{\delta_i(\gamma)}{2}| n_{i,t} = l,B_t>0\}
	\leq  \sum_{l=\lceil L_i \rceil}^{\infty}\exp\{-2l(\frac{\delta_i(\gamma)}{2})^2\}\quad(\textrm{By Fact \ref{fact:hoeffdinginequality}})\\
	\le&\int_{L_i-1}^{\infty}\exp\{-\frac{1}{2}\iota\delta_i^2(\gamma)\}\mathrm{d}\iota = \frac{2e^{\frac{1}{2}\delta_i^2(\gamma)}}{B\delta_i^2(\gamma)}.
	\end{aligned}
	\label{tmp:___r_1}
	\end{equation}
\end{small}
For the second term of \myeqref{eq:prAY2}, we have
\begin{small}
	\begin{equation}
	\begin{aligned}
	& \myprob\{\theta_i^r(t) > \mu_i^r + \delta_i(\gamma),A_i^r(t),n_{i,t} \geq \lceil L_i\rceil|B_t>0\}
	\leq  \myprob\{\theta_i^r(t) > \frac{S^r_i(t)}{n_{i,t}} + \frac{\delta_i(\gamma)}{2},n_{i,t} \geq \lceil L_i\rceil|B_t>0\} \\
	\leq & \sum_{l=\lceil L_i \rceil }^{\infty}\myprob\{\theta_i^r(t) > \frac{S^r_i(t)}{n_{i,t}} + \frac{\delta_i(\gamma)}{2},n_{i,t} =l|B_t>0\}
	\leq  \sum_{l=\lceil L_i \rceil }^{\infty}\myprob\{\theta_i^r(t) > \frac{S^r_i(t)}{l} + \frac{\delta_i(\gamma)}{2} | n_{i,t}=l,B_t>0\}\\
	= & \sum_{l=\lceil L_i \rceil }^{\infty}\myexp[F_{l+1 , \frac{S^r_i(t)}{l} + \frac{\delta_i(\gamma)}{2}}^B(S^r_i(t))]\qquad\textrm{(By Fact \ref{fact:beta_binomial})}\\
	\leq & \sum_{l=\lceil L_i \rceil }^{\infty}\myexp[F_{l , \frac{S^r_i(t)}{l} + \frac{\delta_i(\gamma)}{2}}^B(S^r_i(t))]\qquad\textrm{(By Fact \ref{fact:Binomial_inequality} )}\\
	\leq & \sum_{l=\lceil L_i \rceil }^{\infty}\exp\{-2l(\frac{\delta_i(\gamma)}{2})^2\}\qquad\textrm{(By Fact \ref{fact:Binomial_hoeff})}\\
	\le&\int_{L_i-1}^{\infty}\exp\{-\frac{1}{2}\iota\delta_i^2(\gamma)\}\mathrm{d}\iota = \frac{2e^{\frac{1}{2}\delta_i^2(\gamma)}}{B\delta_i^2(\gamma)}.
	\end{aligned}
	\label{tmp:__r_2}
	\end{equation}
\end{small}

Therefore, according to \myeqref{eq:prAY2}, \myeqref{tmp:___r_1} and \myeqref{tmp:__r_2}, we have
\begin{small}
	\begin{equation}
	\myprob(\overline{A_i^r(t)} , n_{i,t} \geq \lceil L_i\rceil|B_t>0) \leq \frac{4e^{\frac{1}{2}\delta_i^2(\gamma)}}{B\delta_i^2(\gamma)} \le \frac{7}{B\delta_i^2(\gamma)}.
	\end{equation}
\end{small}

\subsection{Derivation of inequality \myeqref{eq:bound_2ndEqn}}\label{subapp:proofoflowprobevent12}
Define $A_i^c(t)$ as the event:
$A_i^c(t):\;\frac{S_i^c(t)}{n_{i,t}} \geq \mu_i^c - \frac{\delta_i(\gamma)}{2}$.
We know that
\begin{equation}
\begin{aligned}
& \myprob\{\overline{E^c_i(t)} , n_{i,t} \geq \lceil L_i\rceil|B_t>0\} = \myprob\{\theta_i^c(t) < \mu_i^c - \delta_i(\gamma),n_{i,t} \geq \lceil L_i\rceil|B_t>0\} \\
=  & \myprob\{\theta_i^c(t) < \mu_i^c - \delta_i(\gamma),A_i^c(t),n_{i,t} \geq \lceil L_i\rceil|B_t>0\}  + \myprob\{\theta_i^c(t) < \mu_i^c - \delta_i(\gamma), \overline{A_i^c(t)},n_{i,t} \geq \lceil L_i\rceil|B_t>0\} \\
\leq & \myprob\{\overline{A_i^c(t)},n_{i,t} \geq \lceil L_i\rceil|B_t>0\}  + \myprob(\theta_i^c(t) < \mu_i^c - \delta_i(\gamma),A_i^c(t),n_{i,t} \geq \lceil L_i\rceil|B_t>0). \\
\end{aligned}
\label{eq:prAZ2}
\end{equation}
We can obtain
\begin{small}
	\begin{equation}
	\begin{aligned}
	& \myprob(\overline{A_i^c(t)} , n_{i,t} \geq \lceil L_i\rceil|B_t>0)
	\leq \sum_{l=\lceil L_i \rceil}^{\infty}\myprob(\overline{A_i^c(t)} , n_{i,t} =l|B_t>0) \\
	\leq & \sum_{l=\lceil L_i\rceil}^{\infty}\myprob( S_i^c(t) - n_{i,t}\mu_i^c \leq -n_{i,t}\frac{\delta_i(\gamma)}{2} |B_t>0, n_{i,t} = l)
	=  \sum_{l=\lceil L_i\rceil}^{\infty}\myprob( S_i^c(t) - l\mu_i^c \leq -l\frac{\delta_i(\gamma)}{2}|B_t>0,  n_{i,t} = l) \\
	\leq & \sum_{l=\lceil L_i\rceil}^{\infty}\exp\{-2l(\frac{\delta_i(\gamma)}{2})^2\}
	\leq \int_{L_i-1}^{\infty}\exp\{-\frac{1}{2}\iota\delta_i^2(\gamma)\}\mathrm{d}\iota\leq\frac{2e^{\frac{1}{2}}}{B\delta_i^2(\gamma)}.	
	\end{aligned}
	\label{tmp_lowprob_1}
	\end{equation}
\end{small}
For the second term in \myeqref{eq:prAZ2}, we have
\begin{small}
	\begin{equation}
	\begin{aligned}
	& \myprob(\theta_i^c(t) < \mu_i^c - \delta_i(\gamma),A_i^c(t),n_{i,t} \geq \lceil L_i\rceil|B_t>0) \\
	\leq& \sum_{l = \lceil L_i\rceil}^{\infty}\myprob(\theta_i^c(t) < \mu_i^c - \delta_i(\gamma),A_i^c(t),n_{i,t} =l|B_t>0)
	\leq  \sum_{l=\lceil L_i\rceil}^{\infty}\myprob(\theta_i^c(t) < \frac{S^c_i(t)}{l} - \frac{\delta_i(\gamma)}{2} | n_{i,t}=l , B_t>0)\\
	= & \sum_{l=\lceil L_i\rceil}^{\infty}\myexp(1 - F_{l+1 , \frac{S^c_i(t)}{l} - \frac{\delta_i(\gamma)}{2}}^B(S^c_i(t)))\qquad\textrm{(By Fact \ref{fact:beta_binomial})}\\
	\leq & \sum_{l=\lceil L_i\rceil}^{\infty}\exp\{2\delta_i(\gamma) - \frac{1}{2}l\delta_i^2(\gamma)\}\qquad\textrm{(By Fact \ref{fact:Binomial_hoeff} , $\triangle$)}\\
	\leq & \exp\{2\delta_i(\gamma)\}\int_{L_i-1}^{\infty}\exp\{-\frac{1}{2}\iota\delta_i^2(\gamma)\}\mathrm{d}\iota\leq\frac{2e^{\frac{5}{2}}}{B\delta_i^2(\gamma)}.
	\end{aligned}
	\label{tmp_lowprob_2}
	\end{equation}
\end{small}
Note that if $B > e$, $L_i > \frac{2}{\delta_i(\gamma)}$, then we can apply Fact \ref{fact:Binomial_hoeff}. Usually $B$ is very large in bandit setting and we can set $B > e$. Accordingly, the formula marked with ($\triangle$) holds.

Therefore, according to \myeqref{eq:prAZ2}, \myeqref{tmp_lowprob_1} and \myeqref{tmp_lowprob_2}, we have
\begin{small}
	\begin{equation}
	\myprob\{\overline{E^c_i(t)} , n_{i,t} \geq \lceil L_i\rceil|B_t>0\} \leq \frac{2e^{\frac{1}{2}}}{B\delta_i^2(\gamma)}  + \frac{2e^{\frac{5}{2}}}{B\delta_i^2(\gamma)}\le\frac{28}{B\delta_i^2(\gamma)}.
	\end{equation}
\end{small}

\subsection{Derivation of inequality \myeqref{eq:boundthefirsttermAB}}
\begin{small}
	\begin{equation}
	\begin{aligned}
	&\sum_{t=1}^{\infty}\myprob\{\overline{E_i^\theta(t)},n_{i,t}\ge \lceil L_i \rceil,B_t > 0 \}= \sum_{t=1}^{\infty}\myprob\{\overline{E_i^\theta(t)},n_{i,t}\ge \lceil L_i \rceil|B_t > 0 \}\myprob\{B_t>0\}\\
	\leq &  \frac{35}{B\delta_i^2(\gamma)}\sum_{t=1}^{\infty}\myprob\{B_t>0\}
	\leq \frac{35}{\delta_i^2(\gamma)\mu^c_{\min}}.
	\end{aligned}
	\end{equation}
\end{small}

\subsection{Derivation of inequality \myeqref{eq:bound_arm1_hard} }\label{app:veryhardderivation}
The derivation of \myeqref{eq:bound_arm1_hard} can be decomposed into three steps:\\
\emph{$\mathcal{S}$tep A: Bridge the probability of pulling arm $1$ and that of pulling arm $i$ $\forall i > 1$ as follows:}
\begin{small}
\begin{equation}
\myprob\{I_t=i|E_i^\theta(t),H_{t-1},B_t > 0\}
\leq\frac{1-p_{i,t}}{p_{i,t}}\myprob\{I_t=1|E_i^\theta(t),H_{t-1},B_t>0\}.
\label{lemma:relation_armi_arm1_prob}
\end{equation}
\end{small}
\emph{Proof:} 	
Define $\varrho_i = \frac{\mu_1^r - \epsilon_i(\gamma)}{\mu_1^c + \epsilon_i(\gamma)}$. Note that throughout this proof, all the probabilities are conditioned on $B_t>0$. That is, $\myprob\{\cdot|\cdot\}$ should be $\myprob\{\cdot|\cdot, B_t>0\}$. We have
\begin{small}
	\begin{equation*}
	\myprob\{I_t=i|E_i^\theta(t),H_{t-1}\}\leq\myprob\{\theta_j(t)\leq \varrho_i \;\forall j| E_i^\theta(t),H_{t-1}\}.
	\end{equation*}
\end{small}
Given the history $H_{t-1}$, the random variables $\theta_j(t)$ $\forall j\in[K]$ are independent. Thus,
\begin{small}
	\begin{equation*}
	\begin{aligned}
	&\myprob\{\theta_j(t)\leq \varrho_i \;\forall j\in[K]| E_i^\theta(t),H_{t-1}\} \\
	=& \myprob\{\theta_1(t)\leq \varrho_i| E_i^\theta(t),H_{t-1}\}\myprob\{\theta_j(t)\leq \varrho_i \;\forall j\ne 1| E_i^\theta(t),H_{t-1}\}\\
	=& \myprob\{\theta_1(t)\leq \varrho_i| H_{t-1}\}\myprob\{\theta_j(t)\leq \varrho_i \;\forall j\ne 1| E_i^\theta(t),H_{t-1}\}\\
	=&(1-p_{i,t})\myprob\{\theta_j(t)\leq \varrho_i \;\forall j\ne 1| E_i^\theta(t),H_{t-1}\}.
	\end{aligned}
	\end{equation*}
\end{small}
Furthermore, we have
\begin{small}
	\begin{equation*}
	\begin{aligned}
	&\myprob\{I_t=1|E_i^\theta(t),H_{t-1}\} \\
	\geq&\myprob\{\theta_1(t) > \varrho_i \geq \theta_j(t) \;\forall j \ne 1|E_i^\theta(t),H_{t-1}\}\\
	\geq&\myprob\{\theta_1(t) > \varrho_i |E_i^\theta(t),H_{t-1}\} \myprob\{ \varrho_i \geq \theta_j(t) \;\forall j \ne 1|E_i^\theta(t),H_{t-1}\}\\
	\geq&\myprob\{\theta_1(t) > \varrho_i |H_{t-1}\} \myprob\{\theta_j(t)\leq\varrho_i \;\forall j \ne 1|E_i^\theta(t),H_{t-1}\}\\
	=&p_{i,t}\myprob\{ \theta_j(t)\leq\varrho_i \;\forall j \ne 1|E_i^\theta(t),H_{t-1}\}.
	\end{aligned}
	\end{equation*}
\end{small}
Therefore, we can conclude that
\begin{small}
	\begin{equation*}
	\begin{aligned}
	&\myprob\{I_t=i|E_i^\theta(t),H_{t-1}\}\leq\myprob\{\theta_j(t)\leq \varrho_i \;\forall j| E_i^\theta(t),H_{t-1}\}\\
	\leq&(1-p_{i,t})\myprob\{\theta_j(t)\leq \varrho_i \;\forall j\ne 1| E_i^\theta(t),H_{t-1}\}\\
	\leq&\frac{1-p_{i,t}}{p_{i,t}}\myprob\{I_t=1|E_i^\theta(t),H_{t-1}\}.\qquad\square
	\end{aligned}
	\end{equation*}
\end{small}
\emph{$\mathcal{S}$tep B: Prove the intermediate step in inequality \myeqref{eq:tmp_bound_secondevent_a}}\\
\begin{small}
\begin{equation}
		\myprob\{I_t=i,E_i^\theta(t)|B_t>0\}
		\leq  \myexp\Big\{\frac{1-p_{i,t}}{p_{i,t}}\myprob\{I_t=1|H_{t-1},B_t>0\}\Big\}.
		\label{eq:tmp_bound_secondevent_a}
\end{equation}
\end{small}
\emph{Proof:} Note throughout this proof, all the probabilities are conditioned on $B_t>0$. That is, $\myprob\{\cdot|\cdot\}$ should be $\myprob\{\cdot|\cdot, B_t>0\}$.
\begin{small}
	\begin{align}
	&\myprob\{I_t=i,E_i^\theta(t)\}=\myexp\{\myprob\{I_t=i,E_i^\theta(t)|H_{t-1}\}\}\qquad(\textrm{The expectation is taken w.r.t. $H_{t-1}$.})\nonumber\\
	=&\myexp\{\myprob\{I_t=i|E_i^\theta(t),H_{t-1}\}\myprob\{E_i^\theta(t)|H_{t-1}\}\}\nonumber\\
	\leq &\myexp\Big\{\frac{1-p_{i,t}}{p_{i,t}}\myprob\{I_t=1|E_i^\theta(t),H_{t-1}\}\myprob\{E_i^\theta(t)|H_{t-1}\}\Big\}\qquad(\textrm{obtained by \myeqref{lemma:relation_armi_arm1_prob}})\nonumber\\
	= & \myexp\Big\{\frac{1-p_{i,t}}{p_{i,t}}\myprob\{I_t=1,E_i^\theta(t)|H_{t-1}\}\Big\}\nonumber\\
	\leq & \myexp\Big\{\frac{1-p_{i,t}}{p_{i,t}}\myprob\{I_t=1|H_{t-1}\}\Big\}.\qquad\qquad\square\nonumber
	\end{align}
\end{small}
\emph{$\mathcal{S}$tep C: Derivation of  inequality \myeqref{eq:bound_arm1_hard}}\\
\emph{Proof:}
\begin{small}
\begin{align}
&\sum_{t=1}^{\infty}\myprob\{I_t=i,E_i^\theta(t),B_t>0\} \leq\sum_{t=1}^{\infty}\myprob\{I_t=i,E_i^\theta(t)|B_t>0\} \nonumber\\
\leq&\sum_{t=1}^{\infty}\myexp\Big\{\frac{1-p_{i,t}}{p_{i,t}}\myprob\{I_t=1|H_{t-1},B_t>0\}\Big\}\qquad(\textrm{obtained by \myeqref{eq:tmp_bound_secondevent_a}})\nonumber\\
\leq&\sum_{k=0}^{\infty}\myexp\sum_{t=\tau_k+1}^{\tau_{k+1}}\Big\{\frac{1-p_{i,t}}{p_{i,t}}\myprob\{I_t=1|H_{t-1},B_t>0\}\Big\}\qquad(\textrm{divide the rounds $\{1,2,\cdots\}$ into blocks $\{[\tau_{k}+1,\tau_{k+1}]\}_{k=0}^{\infty}$})\nonumber\\
\leq&\sum_{k=0}^{\infty}\myexp\Big\{\frac{1-p_{i,\tau_k+1}}{p_{i,\tau_k+1}}\sum_{t=\tau_k+1}^{\tau_{k+1}}\myprob\{I_t=1|H_{t-1},B_t>0\}\Big\}\qquad(\textrm{$p_{i,t}$ does not change in the period $[\tau_k+1,\tau_{k+1}]$})\nonumber\\
\leq&\sum_{k=0}^{\infty}\myexp\bigg\{\frac{1-p_{i,\tau_k+1}}{p_{i,\tau_k+1}}\bigg\}\qquad(\textrm{during $[\tau_k+1,\tau_{k+1}]$, arm $1$ is pulled only once at round $\tau_{k+1}$})\nonumber\\
\leq&\sum_{k=0}^{\infty}\Big(\myexp\Big\{\frac{1}{p_{i,\tau_k+1}}\Big\}-1\Big). \qquad\qquad\square\nonumber
\end{align}
\end{small}

\subsection{Derivation of inequality \myeqref{lemma:bound_reward_for_arm1}}
\label{subapp:boundcomplex1}
In this subsection, we will bound $\myexp[\frac{1}{p^r_{i , \tau_k + 1}}]$. We divide the set $\{0,1,\cdots,k\}$ into four subsets: (i) $[0, \lfloor y_ik\rfloor-1]$; (ii) $[ \lfloor y_ik\rfloor , \lceil y_ik \rceil]$; (iii)  $[\lceil y_ik \rceil+1 , \lfloor \mu_1^rk - \frac{\epsilon_i}{2}k \rfloor]$; (iv) $[\lfloor \mu_1^rk - \frac{\epsilon_i}{2}k \rfloor + 1 , k]$. We will bound $\myexp[\frac{1}{p^r_{i,\tau_k+1}}]$ in the four subsets. Note
\begin{equation}
\myexp[\frac{1}{p^r_{i,\tau_k+1}}] = \sum_{s = 0}^{k}\frac{f_{k , \mu_1^r}(s)}{1-F_{s  + 1 , k - s + 1}^{Beta}(\mu_1^r - \epsilon_i)}= \sum_{s = 0}^{k}\frac{f_{k , \mu_1^r}(s)}{F^B_{k + 1 , y_i}(s)},
\end{equation}
where $f_{k,\mu^r_1}(s)$ represents the probability that exactly $s$ out of $k$ Bernoulli trials succeed with success probability $\mu^r_1$ in a single trial.\\
(\emph{Case i})  $s\in[0, \lfloor y_ik\rfloor-1]$: First, $\forall s$, we have
\begin{small}
\begin{equation}
\begin{aligned}
&\frac{f_{k,\mu_1^r}(s)}{F^B_{k+1,y_i}(s)} \leq \Theta\left(\frac{f_{k,\mu_1^r}(s)}{\frac{y_i(k+1-s)}{y_i(k+1)-s}\binom{k+1}{s}y_i^s(1-y_i)^{k+1-s}}\right) + \Theta(1)f_{k,\mu_1^r}(s)\\
=&\Theta\left( \frac{y_i(k+1)-s}{y_i(k+1)}\frac{(\mu_1^r)^s(1 - \mu_1^r)^{k-s}}{y_i^s(1-y_i)^{k-s}}\frac{1}{1-y_i} \right) + \Theta(1)f_{k,\mu_1^r}(s)\\
=&\Theta\left( \frac{y_i(k+1)-s}{y_i(1-y_i)(k+1)}(\frac{1 - \mu_1^r}{1 - y_i})^kR_{1,i}^s \right) + \Theta(1)f_{k,\mu_1^r}(s)\\
\end{aligned}
\end{equation}
\end{small}
One can verify that $[\frac{\mu_1^r(1-y_i)}{y_i(1-\mu_1^r)})]^{y_i}\frac{1-\mu_1^r}{1-y_i}=e^{-D_{1,i}}$. Note that $R_{1,i}>1$ $\forall i > 1$. Then,
\begin{equation}
\begin{aligned}
&\frac{(\frac{1 - \mu_1^r}{1 - y_i})^k}{y_i(1-y_i)(k+1)}\sum_{s=0}^{\lfloor y_ik\rfloor-1}(y_i(k+1)-s)R_{1,i}^s\\
=&\frac{(\frac{1 - \mu_1^r}{1 - y_i})^k}{y_i(1-y_i)(k+1)}\frac{y_i(k+1)(R_1^{\lfloor y_ik \rfloor}-1)(R_{1,i}-1) - (\lfloor y_ik \rfloor-1)R_{1,i}^{\lfloor y_ik\rfloor+1} - R_{1,i} + \lfloor y_ik \rfloor R_{1,i}^{\lfloor yk \rfloor}}{(R_{1,i}-1)^2} \\
\leq&\frac{(\frac{1 - \mu_1^r}{1 - y_i})^k}{y_i(1-y_i)(k+1)}\frac{y_i(k+1)R_{1,i}^{\lfloor y_ik \rfloor}(R_{1,i}-1) - (\lfloor y_ik \rfloor-1)R_{1,i}^{\lfloor yk\rfloor+1}  + \lfloor y_ik \rfloor R_{1,i}^{\lfloor y_ik \rfloor}}{(R_{1,i}-1)^2} \\
=&\frac{(\frac{1 - \mu_1^r}{1 - y_i})^kR_{1,i}^{\lfloor y_ik \rfloor}}{y_i(1-y_i)(k+1)}\frac{y_i(k+1)(R_{1,i}-1) - (\lfloor y_ik \rfloor-1)R_{1,i}  + \lfloor y_ik \rfloor}{(R_{1,i}-1)^2} \\
\leq&\frac{(\frac{1 - \mu_1^r}{1 - y})^kR_{1,i}^{\lfloor y_ik \rfloor}}{y_i(1-y_i)(k+1)}\frac{3R_{1,i}}{(R_{1,i}-1)^2}
\leq\frac{3R_{1,i}e^{-D_{1,i}k}}{y_i(1-y_i)(k+1)(R_{1,i}-1)^2}.
\end{aligned}
\end{equation}
For the latter part, i.e.,$\sum_{s=0}^{\lfloor y_ik \rfloor - 1}\Theta(1)f_{k , \mu_1^r}(s) $, it can be seen as the probability that there are less than $\lfloor y_ik \rfloor$ successful trials in a $k$-trial Bernoulli experiment. Denote the experiment result of trial $i(\in[k])$ as $X_i$ and $X_i \sim \mathcal{B}(\mu_1^r)$. $\{X_i\}_{i=1}^{k}$ are independent and identically distributed. We can conclude that
\begin{equation}
\sum_{s=0}^{\lfloor y_ik \rfloor - 1}\Theta(1)f_{k , \mu_1^r}(s) \leq \Theta(1)\myprob\{X_1 + X_2 + \cdots + X_k \leq y_ik - 1 \leq y_ik\} \leq \Theta(1)\exp\{-2k(y_i - \mu_1^r)^2\} = \Theta(e^{-2\epsilon_i^2k}).
\end{equation}
(\emph{Case ii}) $s \in[ \lfloor y_ik\rfloor , \lceil y_ik \rceil]$:
\begin{equation}
\begin{aligned}
&\sum_{s=\lfloor y_ik\rfloor}^{ \lceil y_ik \rceil}\frac{f_{k , \mu_1^r}(s)}{F^B_{k + 1 , y_i}(s)}\leq\sum_{s=\lfloor y_ik\rfloor}^{ \lceil y_ik \rceil}\frac{f_{k,\mu_1^r}(s)}{f_{k+1,y_i}(s)} = \sum_{s=\lfloor y_ik\rfloor}^{ \lceil y_ik \rceil}\frac{k - s+1}{k+1}\frac{1}{1-y_i}[\frac{\mu_1^r(1-y_i)}{y_i(1 - \mu_1^r)}]^s(\frac{1 -  \mu_1^r}{1 - y_i})^k \\
\leq&  \sum_{s=\lfloor y_ik\rfloor}^{ \lceil y_ik \rceil}\frac{1}{1-y_i}R_1^s(\frac{1 -  \mu_1^r}{1 - y_i})^k \leq \frac{1}{1-y_i}R_{1,i}^{y_ik}(\frac{1 -  \mu_1^r}{1 - y_i})^k(1 + R_{1,i}) \leq \frac{1 +R_{1,i}}{1 - y_i}e^{-D_{1,i}k}.
\end{aligned}
\end{equation}
(\emph{Case iii}) $s\in[\lceil y_ik \rceil + 1 , \lfloor\mu_1^rk - \frac{\epsilon_i}{2}k \rfloor]$: One can verify that $s \ge y_i(k+1)-\sqrt{(k+1)y_i(1-y_i)}$. Thus, we have $F^B_{k + 1 , y_i}(s)=\Theta(1)$. Denote $X_i\sim\mathcal{B}(\mu_1^r)$ $\forall i\in[k]$ and $\{X_i\}_{i=1}^{k}$  are independent and identically distributed.
\begin{equation}
\begin{aligned}
&\sum_{s=\lceil y_ik \rceil + 1}^{ \lfloor\mu_1^rk - \frac{\epsilon_i}{2}k \rfloor}\frac{f_{k , \mu_1^r}(s)}{F^B_{k+1,y_i}(s)} = \Theta\left( \sum_{s = \lceil y_ik \rceil + 1}^{ \lfloor\mu_1^rk - \frac{\epsilon_i}{2}k \rfloor}f_{k , \mu_1^r}(s) \right) \leq\Theta(1)\myprob\{X_1 + X_2 + \cdots + X_k \leq  \lfloor\mu_1^rk - \frac{\epsilon_i}{2}k \rfloor\} \\
\leq& \Theta(1)\myprob\{X_1 + X_2 + \cdots + X_k \leq  \mu_1^rk - \frac{\epsilon_i}{2}k \}\leq \Theta(e^{-\frac{1}{2}k\epsilon_i^2}).
\end{aligned}
\end{equation}
(\emph{Case iv}) $s\in[\lfloor \mu_1^rk - \frac{\epsilon_i}{2}k \rfloor + 1 , k]$: denote $X_i\sim\mathcal{B}(y_i)$ $\forall i \in [k+1]$ and $\{X_i\}_{i=1}^{k+1}$ are independent and identically distributed. We have that
\begin{equation}
\begin{aligned}
& 1 - F^B_{k + 1 , y_i}(s) \leq \myprob\{X_1 + X_2 + \cdots + X_{k + 1} \geq \lfloor \mu_1^rk - \frac{\epsilon_i}{2}k \rfloor + 2 \} \\
\leq & \myprob\{X_1 + X_2 + \cdots + X_{k + 1} \geq  y_ik + \frac{\epsilon_i}{2}k + y_i \} \leq \exp\{-\frac{\epsilon_i^2k^2}{2(k+1)}\}
\end{aligned}
\end{equation}
Thus we have that
\begin{equation}
\sum_{s = \lfloor \mu_1^rk - \frac{\epsilon_i}{2}k \rfloor + 1}^{k}\frac{f_{k , \mu_1^r}(s)}{F^B_{k+1 , y_i}(s)} \leq \sum_{s = \lfloor \mu_1^rk - \frac{\epsilon_i}{2}k \rfloor + 1}^{k}\frac{f_{k , \mu_1^r}(s)}{1 - \exp\{-\frac{\epsilon_i^2k^2}{2(k+1)}\}} \leq \frac{1}{1 - \exp\{-\frac{\epsilon_i^2k^2}{2(k+1)}\}}=1 + \frac{1}{\exp\{\frac{\epsilon_i^2k^2}{2(k+1)}\}-1}.
\end{equation}

Therefore, we can conclude that
\begin{equation}
\myexp[\frac{1}{p^r_{i , \tau_k + 1}}]\leq 1 + \Theta\bigg(\frac{3R_{1,i}e^{-D_{1,i}k}}{y_i(1-y_i)(k+1)(R_{1,i}-1)^2} + e^{-2\epsilon_i^2k} + \frac{1 +R_{1,i}}{1 - y}e^{-D_{1,i}k} + e^{-\frac{1}{2}k\epsilon_i^2} +  \frac{1}{\exp\{\frac{\epsilon_i^2k^2}{2(k+1)}\}-1}\bigg).
\end{equation}

\subsection{Derivation  of inequality \myeqref{lemma:bound_cost_for_arm1}}
\label{subapp:boundcomplex2}
	If $z_i \ge 1$, we have  $\myexp[\frac{1}{p^c_{i , \tau_k + 1}}] =1$ and \myeqref{lemma:bound_cost_for_arm1} holds trivially. If $z_i < 1$, we get
		\begin{small}
			\begin{equation*}
				\myexp[\frac{1}{p^c_{i,\tau_k+1}}] = \sum_{s = 0}^{k}\frac{f_{k , \mu_1^c}(s)}{F_{s  + 1 , k - s + 1}^{Beta}(\mu_1^c + \epsilon_i)}= \sum_{s = 0}^{k}\frac{f_{k , \mu_1^c}(s)}{1 - F^B_{k + 1 , z_i}(s)},
			\end{equation*}
		\end{small}
where $f_{k,\mu^c_1}(s)$ represents the probability that exactly $s$ out of $k$ Bernoulli trials succeed with success probability $\mu^c_1$ in a single trial. We divide the set $\{0,1,\cdots,k\}$ into four subsets: (i) $[0, \lfloor\mu_1^ck+\frac{\epsilon_i}{2}k\rfloor]$, (ii) $[\lceil\mu_1^ck + \frac{\epsilon_i}{2}k\rceil,\lfloor z_ik \rfloor-1]$, (iii) $\lfloor z_ik \rfloor$, and (iv) $[\lceil z_ik \rceil ,k]$, and then bound $\myexp[\frac{1}{p^c_{i,\tau_k+1}}]$ in the four subsets as follows.\\
(Case i) If $s \leq \lfloor\mu_1^ck+\frac{\epsilon_i}{2}k\rfloor$, denote $X_i\sim\mathcal{B}(z_i)$ $\forall i \in [k+1]$ and $\{X_i\}_{i=1}^{k+1}$  are independent and identically distributed. We have
		\begin{small}
			\begin{equation*}
				\begin{aligned}
					 F^B_{k+1 , z_i}(s) \leq \myprob\{X_1 + X_2 + \cdots + X_{k+1} \leq s \leq \mu_1^ck+\frac{\epsilon_i}{2}k\}
					 \leq \myprob\{X_1 + X_2 + \cdots + X_{k+1} \leq z_ik-\frac{\epsilon_i}{2}k + z_i\}\leq\exp\{-\frac{\epsilon_i^2k^2}{2(k+1)}\},
				\end{aligned}
			\end{equation*}
		\end{small}
Therefore,
		\begin{small}
			\begin{equation*}
				\begin{aligned}
					\sum_{s = 0}^{\lfloor\mu_1^ck+\frac{\epsilon_i}{2}k\rfloor}\frac{f_{k , \mu_1^c}(s)}{1 - F^B_{k + 1 , z_i}(s)}\leq\sum_{s = 0}^{\lfloor\mu_1^ck+\frac{\epsilon_i}{2}k\rfloor}\frac{f_{k , \mu_1^c}(s)}{1 - \exp\{-\frac{\epsilon_i^2k^2}{2(k+1)}\}}
					\leq\frac{1}{1 -\exp\{-\frac{\epsilon_i^2k^2}{2(k+1)}\}}\leq1+\frac{1}{\exp\{\frac{\epsilon_i^2k^2}{2(k+1)}\}-1}.
				\end{aligned}
			\end{equation*}
		\end{small}
(Case ii) We can verify that $\forall s\in[\lceil\mu_1^ck + \frac{\epsilon_i}{2}k\rceil,\lfloor z_ik \rfloor-1]$,  $k-s\geq(1-z_i)(k+1)-\sqrt{(k+1)z_i(1-z_i)}$, and thus $1 - F^B_{k+1,z_i}(s)=\Theta(1)$. Then similar to (Case i),
denote $X_i\sim\mathcal{B}(\mu_1^c)$ ($\forall i \in [k]$) and $\{X_i\}_{i=1}^{k}$  are independent and identically distributed. We have
\begin{small}
		\begin{equation}
		\sum_{s =\lceil\mu_1^ck + \frac{\epsilon_i}{2}k\rceil }^{\lfloor z_ik \rfloor-1}\frac{f_{k , \mu_1^c}(s)}{1 - F^B_{k + 1 , z_i}(s)}=\Theta\left(\sum_{s =\lceil\mu_1^ck + \frac{\epsilon_i}{2}k\rceil }^{\lfloor z_ik \rfloor-1}f_{k , \mu_1^c}(s)\right)\leq\Theta\left(\myprob\{X_1 + X_2 + \cdots + X_k \geq \mu_1^ck + \frac{\epsilon_i}{2}k\}\right)\leq\Theta\left(\exp\{-\frac{\epsilon_i^2}{2}k\}\right).
		\end{equation}
\end{small}
(Case iii) One can verify that $[\frac{\mu_1^c(1-z_i)}{z_i(1-\mu_1^c)})]^{z_i}\frac{1-\mu_1^c}{1-z_i}=e^{-D_{2,i}}$. Then, with some simple derivations, we can get
		\begin{small}
		\begin{equation}
		\frac{f_{k , \mu_1^c}(s)}{1 - F^B_{k+1,z_i}(s)} \leq \frac{f_{k , \mu_1^c}(s)}{f_{k+1,z_i}(s+1)}=\frac{s+1}{k+1}\frac{(\mu_1^c)^s(1-\mu_1^c)^{k-s}}{z_i^{s+1}(1-z_i)^{k-s}}\leq\frac{1}{z_i}R_{2,i}^s[\frac{1-\mu_1^c}{1-z_i}]^k\leq\frac{1}{z_iR_{2,i}}R_{2,i}^{zk}[\frac{1-\mu_1^c}{1-z_i}]^k\leq\frac{1}{z_iR_{2,i}}e^{-D_{2,i}k}.
		\end{equation}
		\end{small}
(Case iv) For any $s\in[\lceil z_ik \rceil ,k]$, $\frac{f_{k , \mu_1^c}(s)}{1 - F^B_{k+1,z_i}(s)}$ is bounded by
		\begin{small}
			\begin{equation*}
				\begin{aligned}
					& \Theta\left(\frac{f_{k,\mu_1^c}(s)}{\frac{(1-z_i)(s+1)}{(1-z_i)(k+1)-k+s}\binom{k+1}{k-s}(1-z_i)^{k-s}z_i^{s+1}}\right) + \Theta(1)f_{k , \mu_1^c}(s)
					 = \Theta\left( \frac{(1-z_i)(k+1)-k+s}{z_i(1-z_i)(k+1)}R_{2,i}^s(\frac{1 - \mu_1^c}{1 - z_i})^k \right)+ \Theta(1)f_{k , \mu_1^c}(s).
				\end{aligned}
			\end{equation*}
		\end{small}
		Note $R_{2,i} < 1$. The first term of the r.h.s of the above equation can be upper bounded by
		\begin{small}
			\begin{equation*}
				\begin{aligned}
					& \frac{1}{z_i}(\frac{1 - \mu_1^c}{1 - z_i})^k\sum_{s=\lceil z_ik \rceil}^{k}\frac{(1-z_i)(k+1)-k+s}{(1-z_i)(k+1)}R_{2,i}^s\\
					\leq&\frac{1}{z_i}(\frac{1 - \mu_1^c}{1 - z_i})^k\bigg( \frac{1}{k+1}\frac{R_{2,i}^{\lceil z_ik \rceil}}{1 - R_{2,i}} + \frac{\lceil z_ik \rceil R_{2,i}^{\lceil z_ik \rceil}}{(1-z_i)(k+1)(1-R_{2,i})^2}\bigg)\\
					\leq&\frac{1}{z_i}(\frac{1 - \mu_1^c}{1 - z_i})^k\bigg( \frac{1}{k+1}\frac{R_{2,i}^{z_ik}}{1 - R_{2,i}} + \frac{ z_iR_2^{z_ik}}{(1-z_i)(1-R_{2,i})^2} + \frac{R_{2,i}^{z_ik}}{(1-z_i)(k+1)(1-R_{2,i})^2}\bigg)\\
					\leq&e^{-D_{2,i}k}\frac{2 + R_{2,i}(z_i - 1) + zk}{z_i(1-z_i)(k+1)(1-R_{2,i})^2}\leq \frac{2e^{-D_{2,i}k}}{z_i(1-z_i)(1-R_{2,i})^2}.
				\end{aligned}
			\end{equation*}
		\end{small}
				
		Similar to the analysis of case (i), we can obtain
		\begin{small}
		\begin{equation}
		\begin{aligned}
		&\sum_{s=\lceil z_ik \rceil}^{k}\Theta(1)f_{k , \mu_1^c}(s) \leq \Theta(1)\myprob\{X_1 + X_2 + \cdots + X_k \geq \lceil z_ik \rceil\} \leq \Theta\{\myprob\{X_1 + X_2 + \cdots + X_k \geq z_ik \}\} \\
		=& \Theta\{\myprob\{X_1 + X_2 + \cdots + X_k -\mu_i^ck\geq z_ik - \mu_1^ck = \epsilon_ik \}\}\leq\Theta(e^{-2\epsilon_i^2k}),
		\end{aligned}
		\end{equation}
		\end{small}
		in which $X_i\sim\mathcal{B}(\mu_1^c)$ $\forall i\in[k]$ and $\{X_i\}_{i=1}^{k}$ are independent and identically distributed.
		
Combining the above analysis, we arrive at inequality \myeqref{lemma:bound_cost_for_arm1}.  $\square$
\subsection{Derivation in the $\mathcal{S}2$ of Subsection \ref{subsec:generalExternsion}}
\label{subapp:boundstopppingtime2}
Note that $\sum_{i=1}^{K}c_k(t)\bm{1}\{I_t=i\}\bm{1}\{B_t>0\}$ is the cost at round $t$. For general bandits, it is possible that the cost at the last round exceeds the left budget. Thus, $\sum_{t=1}^{\infty}\sum_{i=1}^{K}\myexp[c_k(t)\bm{1}\{I_t=i\}\bm{1}\{B_t>0\}]\le B+1$. Therefore, we can obtain
\begin{small}
	\begin{equation}
	\begin{aligned}
	&\sum_{t=1}^{\infty}\myprob\{B_t>0\}\leq\frac{1}{\mu^c_{\min}}\sum_{t=1}^{\infty}\sum_{i=1}^{K}\myexp[c_k(t)\bm{1}\{I_t=i\}|B_t>0]\myprob\{B_t>0\}\\
	\leq & \frac{1}{\mu^c_{\min}}\sum_{t=1}^{\infty}\sum_{i=1}^{K}\myexp[c_k(t)\bm{1}\{I_t=i\}\bm{1}\{B_t>0\}]\leq \frac{B+1}{\mu^c_{\min}}.
	\end{aligned}
	\end{equation}
\end{small}
Therefore, since $B$ is a positive integer, which indicates that $B\ge1$, we can obtain
\begin{small}
\begin{equation}
\sum_{t=1}^{\infty}\myprob\{\overline{E_i^\theta(t)},n_{i,t}\ge \lceil L_i \rceil,B_t > 0 \}\leq\frac{35}{B\delta_i^2(\gamma)}\frac{B+1}{\mu^c_{\min}}
\leq\frac{70}{\delta_i^2(\gamma)\mu^c_{\min}}.
\end{equation}
\end{small}

\subsection{Proof of Remark \ref{remark:comparedwithwenkui}}

The constant in the regret bound of UCB-BV1 \cite{ding2013multi} before $\ln B$  is at least:
\begin{small}
\begin{equation}
	\frac{\mu_1^r}{\mu_1^c}\sum_{i=2}^{K}\Big(\frac{2 + \frac{2}{\mu^c_{\min}} + \Delta_i}{\Delta_i\mu^c_{\min}} \Big)^2 + \sum_{i:\mu_i^r<\mu_1^r}(\mu^r_1-\mu^r_i)\Big(\frac{2 + \frac{2}{\mu^c_{\min}} + \Delta_i}{\Delta_i\mu^c_{\min}} \Big).
\label{eq:before_logB_wenkui}
\end{equation}
\end{small}
While by setting $\gamma=\frac{1}{\sqrt{2}}$ in Theorem \ref{thm:bound_TS_asymp}, the constant before $\ln B$ of our proposed BTS is
\begin{small}
\begin{equation}
\Big(\frac{\mu^r_1}{\mu^c_1}+1\Big)^2\sum_{i=2}^{K}\frac{4}{\mu_i^c\Delta_i},
\label{eq:before_logB_ourbts}
\end{equation}
\end{small}

It is obvious that  $\Delta_i \in (0 , \frac{\mu_1^r}{\mu_1^c})$ $\forall i \ge 2$. We have that $\forall i \ge 2$,
\begin{enumerate}
\item $(2 + \frac{2}{\mu^c_{\min}} + \Delta_i)^2 > (2 + \frac{2}{\mu^c_{\min}})^2 \ge 4(\frac{\mu_1^r}{\mu_1^c}+1)^2$;
\item $\frac{\mu_1^r}{\mu_1^c} > \Delta_i$;
\item $\frac{1}{(\mu^c_{\min})^2} \ge \frac{1}{\mu_i^c}$.
\end{enumerate}
Thus, \myeqref{eq:before_logB_ourbts} is strictly smaller than \myeqref{eq:before_logB_wenkui}.

Similar discussions could be applied to the UCB-BV2 in \cite{ding2013multi}.

\end{document}